\documentclass{article}
\usepackage[a4paper]{geometry}
\usepackage{authblk}

\usepackage[numbers]{natbib}
\usepackage[utf8]{inputenc} %
\usepackage[T1]{fontenc}    %
\usepackage{booktabs}       %
\usepackage{amsfonts}       %
\usepackage{nicefrac}       %
\usepackage{microtype}      %
\usepackage{xcolor}         %
\usepackage{selectp}

\usepackage{xspace,xcolor,enumerate}
\usepackage{amssymb,amsmath}
\usepackage{mathtools}

\usepackage{wrapfig}
\usepackage{pgf, tikz}
\usetikzlibrary{positioning}
\usetikzlibrary{arrows, automata}

\usepackage{bm}
\usepackage{bbm}

\usetikzlibrary{patterns}
\usepackage{amsthm}
\usepackage{thm-restate}
\usepackage{subcaption}
\usepackage{multirow}
\usepackage{hyperref}
\usepackage[capitalize]{cleveref}

\usepackage[linesnumbered, ruled,vlined]{algorithm2e}

\allowdisplaybreaks

\newcommand{\ip}[2] {\ensuremath{\langle #1 , #2 \rangle}}

\newcommand{\norm}[1]{\ensuremath{\lVert #1 \rVert}}
\newcommand{\EE}{\mathbb{E}}
\newcommand{\RR}{\mathbb{R}}

\newcommand{\calS}{\mathcal{S}}

\newcommand{\calX}{\mathcal{X}}

\newcommand{\eps}{\epsilon}
\newcommand{\sig}{\sigma}

\newcommand{\al}{\alpha}
\newcommand{\lda}{\lambda}
\newcommand{\Lda}{\Lambda}

\newcommand{\gam}{\gamma}

\newcommand{\calC}{\mathcal{C}}

\newcommand{\calL}{\mathcal{L}}

\newcommand{\indistribution}{\overset{\mathcal{D}}{=}}
\newcommand{\diag}{\mathrm{Diag}}

\DeclareMathOperator*{\argmin}{argmin}

\newcommand{\one}{\mathbbm{1}}
\newcommand{\p}{\varphi}

\newcommand{\bs}[1]{\boldsymbol{#1}}

\newtheorem{theorem}{Theorem}
\newtheorem{corollary}{Corollary}

\newtheorem{lemma}{Lemma}

\newtheorem{problem}{Problem}

\newtheorem{definition}{Definition}

\newtheorem{remk}{Remark}

\newtheorem{assumption}{Assumption}

\def\showauthornotes{1}
\ifnum\showauthornotes=1
\newcommand{\Authornote}[2]{{\sf\small\color{blue}{[#1: #2]}}}
\else
\newcommand{\Authornote}[2]{}
\fi

\newcommand{\alt}[1]{\widetilde{#1}}
\newcommand{\alts}[1]{\overline{#1}}

\title{Learning Interpretable Concepts: Unifying Causal Representation Learning and Foundation Models}

\author[1]{Goutham Rajendran$^*$}
\author[2,3]{Simon Buchholz\thanks{Equal Contribution}}
\author[4]{\authorcr Bryon Aragam}
\author[2,5]{Bernhard Sch\"{o}lkopf}
\author[1]{Pradeep Ravikumar}
\affil[1]{Machine Learning Dept., Carnegie Mellon University, Pittsburgh, USA}
\affil[2]{Max Planck Institute for Intelligent Systems, T\"ubingen, Germany}
\affil[3]{Tübingen AI Center}
\affil[4]{University of Chicago, Chicago, USA}
\affil[5]{ELLIS Institute Tübingen}

\date{\today}

\begin{document}

\maketitle

\begin{abstract}
To build intelligent machine learning systems, there are two broad approaches. One approach is to build inherently interpretable models, as endeavored by the growing field of causal representation learning. The other approach is to build highly-performant foundation models and then invest efforts into understanding how they work. In this work, we relate these two approaches and study how to learn human-interpretable concepts from data. Weaving together ideas from both fields, we formally define a notion of concepts and show that they can be
provably recovered from diverse data. Experiments on synthetic data and large language models show the utility of our unified approach.
\end{abstract}

\section{Introduction}

A key goal of modern machine learning is to learn representations of complex data that are human-interpretable and can be controlled.
This goal is of paramount importance
given the breadth and importance of ML in today's world.
There seem to be two broad approaches toward such intelligent systems.
The first approach is to build models that are inherently interpretable and then subsequently focus on how to extract maximum performance from them; and the second approach is to build high-performance neural models, and then subsequently invest efforts to understand the inner workings of such models.

A prominent example of the first camp is the field of Causal Representation Learning (CRL) \cite{scholkopf2021towards, scholkopf2022statistical}. CRL is an intricate interplay of ideas from causality, latent variable modeling and deep learning, with the main goal being to reconstruct the true generative factors of data.
To ensure that the true generative factors can be provably
recovered, CRL relies on the central theme of \textit{identifiability} which posits that a unique model fits the data, which in turn implies that the problem of learning the generative factors is well-posed and therefore should theoretically be amenable to modern techniques.
If such a generative model reconstruction can be done, the model will naturally enjoy a host of desired properties such as robustness and generalization.
While this endeavor has been
somewhat successful in many domains such as computer vision \cite{khemakhem2020variational, von2021self, ahuja2023multi}, robotics \cite{lu2021invariant,brehmer2022weakly, lippe2023biscuit, zhang2023interpretable} and genomics \cite{seigal2022linear, zhang2023identifiability}, it is unclear how it relates to the research on foundation models.

The other camp is more empirical, where one tries to build a high-performance model where performance is measured via various downstream tasks and then eventually invest efforts into explaining or interpreting how they work. For instance, large language models and other foundation models are built to be highly performant for a variety of tasks. Owing to their incredible success, there is a growing but heavily-debated belief that such models are truly ``intelligent'' because they have indeed learned the true underlying generative factors somehow, sometimes referred to as the ``world model''. While
we are far from scientifically verifying this, the community has invested tremendous efforts into interpretability research of foundation models, e.g., the field of mechanistic interpretability \cite{olah2022} aims to reverse engineer what large language models learn.

In this work, we take a first step toward unifying these approaches.
We focus on the goal of learning identifiable human-interpretable concepts from complex high-dimensional data.
Specifically, we build a theory of what concepts mean for complex high-dimensional data and then study under what conditions such concepts are identifiable, i.e., when can they be unambiguously recovered from data.
To formally define concepts, we leverage extensive empirical evidence in the foundation model literature that surprisingly shows that, across multiple domains, human-interpretable concepts are often \textit{linearly} encoded in the latent space of such models (see \cref{para: linearity}), e.g., the sentiment of a sentence is linearly represented in the activation space of large language models \cite{tigges2023linear}.
Motivated by this rich empirical literature, we formally define concepts as affine subspaces of some underlying representation space.
Then we connect it to causal representation learning by proving strong identifiability theorems for \textit{only desired concepts} rather than all possible concepts present in the true generative model.
Therefore, in this work we tread the fine line between the rigorous principles of causal representation learning and the empirical capabilities of foundation models, effectively showing
how causal representation learning ideas can be applied to foundation models.

Let us be more concrete. For observed data $X$ that has an underlying representation $Z_u$ with $X = f_u(Z_u)$
for an arbitrary distribution on $Z_u$ and a (potentially complicated) nonlinear underlying mixing map $f_u$, we define concepts as affine subspaces $AZ_u = b$
of the latent space of $Z_u$s, i.e., all observations falling under a concept satisfy an equation of this form.
Since concepts are not precise and can be fuzzy or continuous, we will allow for some noise in this formulation by working with the notion of concept conditional distributions (\cref{def:conditional}).
Of course, in general, $f_u$ and $Z_u$ are very high-dimensional and complex, as they can be used to represent arbitrary concepts.
Instead of ambitiously attempting to reconstruct $f_u$ and $Z_u$ as CRL would do, we go for a more relaxed notion where we attempt to learn a minimal representation that represents \textit{only the subset of concepts we care about}; i.e., a simpler decoder $f$ and representation $Z$---different from $f_u$ and $Z_u$---such that $Z$ linearly captures a subset of relevant concepts as well as a valid representation $X = f(Z)$.
With this novel formulation, we
formally prove that concept learning is identifiable up to simple linear transformations (the linear transformation ambiguity is unavoidable and ubiquitous in CRL).
This relaxes the goals of CRL to only learn relevant representations and not necessarily learn the full underlying model. It further suggests that foundation models do in essence learn such relaxed representations, partially explaining their superior performance for various downstream tasks.

Apart from the above conceptual contribution, we also show that to learn $n$ (atomic) concepts, we only require $n + 2$ environments under mild assumptions. Contrast this with the adage in CRL \cite{khemakhem2020variational, buchholz2023learning} where we require $\dim(Z_u)$ environments for most identifiability guarantees, where as described above we typically have $\dim(Z_u) \gg n + 2$.
These theoretical insights are then validated on synthetic data, where we use a contrastive algorithm to learn such representations for a given collection of concepts.

Moving ahead to real-world data and foundation models, we proceed to show an effective application of our framework to large language models (LLMs). In particular, we consider the alignment problem of making pre-trained LLMs more \textit{truthful}.
First, we make the assumption that pre-trained LLMs have already learnt the concept of truth linearly, as has been empirically observed in \citet{li2023inference} (see also \cref{sec: related_work} for more evidence).
Therefore, we can apply our ideas, which we use to mechanistically reason about how the recent Inference-Time Intervention technique \cite{li2023inference} steers pre-trained LLMs to give more truthful responses to questions.
Then, we exploit our insights to extend this technique, by building steering matrices rather than steering vectors to align LLMs towards concepts. Preliminary experiments using LLaMA \cite{touvron2023llama} show the efficacy of our approach on the TruthfulQA dataset \cite{lin2021truthfulqa}.

In summary, our contributions are:
\begin{enumerate}
    \item
    We formalize the notion of distributions induced by abstract concepts in complex domains such as images or text. Our definition of concept conditional distributions allows both continuous and fuzzy concepts.
    \item We prove near-optimal identifiability results for learning a collection of concepts from a diverse set of environments.
    We also verify our guarantees via a contrastive learning algorithm on synthetic data.
    Thus, our work presents a novel framework for identifying concepts by weaving together ideas from causal representation learning and foundation models.
    \item We show the applicability of our ideas toward mechanistic interpretability, by explaining why inference-time steering vectors align large language models toward abstract concepts such as truthfulness. Furthermore, we extend this idea and propose to use steering matrices instead of steering vectors for better alignment. Our experiments with LLaMA \cite{touvron2023llama} on TruthfulQA \cite{lin2021truthfulqa} show improved performance.
\end{enumerate}

\section{Related work}
\label{sec: related_work}

\paragraph{Causal representation learning} Causal representation learning (CRL) \cite{scholkopf2021towards, scholkopf2022statistical} aims to learn generative factors of high-dimensional data.
This exciting field has seen significant progress in the last few years \cite{khemakhem2020variational, brehmer2022weakly, shen2022weakly, lachapelle2022disentanglement, moran2022identifiable, kivva2022identifiability, buchholz2023learning, gresele2021independent, ahuja2022interventional,von2023nonparametric}.
A fundamental perspective in this field is to ensure that the model parameters we attempt to recover are identifiable \cite{khemakhem2020variational, d2022underspecification, wang2021posterior}. Identifiability is the notion that the model parameters we learn are equivalent to the true model parameters up to simple transformations.
However, it's not clear until this work how this relates to the representations learned by foundation models.
Our work acts as a bridge between these two approaches.
We will elaborate more on the connection of our framework to CRL in \cref{app:crl}.

\paragraph{Linearity of representations}
\label{para: linearity}

Sometimes referred to as the linear representation hypothesis, it is commonly believed that
well-trained foundation models in multiple domains learn linear representations of human-interpretable concepts, with experimental evidence going back at least a decade \cite{mikolov2013linguistic,szegedy2013intriguing, arora2016latent}. This has been experimentally observed in computer vision models \cite{radford2015unsupervised, raghu2017svcca, bau2017network, engel2017latent, kim2018interpretability, wang2023concept, trager2023linear}, language models \cite{mikolov2013linguistic,pennington2014glove,arora2016latent, conneau2018you,tenney2019bert, elhage2022superposition}, large language models \cite{burns2022discovering, tigges2023linear, nanda2023emergent, moschella2022relative, li2023inference, park2023linear, gurnee2023finding, jiang2024learning}, and other intelligent systems \cite{mcgrath2022acquisition, schut2023bridging}.
Various works have also attempted to justify why this happens \cite{levy2014neural, arora2016latent, gittens2017skip, allen2019analogies,ethayarajh2018towards,seonwoo2019additive}.

We take a different angle: Given that this phenomenon has been observed for certain concepts of interest, how does this enable recovery of the concepts themselves?
Consequently, our model assumptions are well-founded and our theory applies to multiple domains of wide interest.

\paragraph{Concepts from pre-trained models}
Concept discovery is an important sub-field of machine learning which attempts to understand the behavior of powerful neural models. We do not attempt to list the numerous experimental works in this direction, see e.g., \citet{schut2023bridging, carvalho2019machine}.
However, theoretical progress in this direction is relatively limited. Prior works have attempted to formalize the notion of concepts \cite{wang2023concept, park2023linear, schut2023bridging}, however their definitions seem specific to the model and domain under consideration, e.g., \citet{park2023linear} focus on binary concepts via large language model representations of counterfactual word pairs, whereas our general concept definitions are applicable to all domains.
Importantly, they do not provide identifiable recovery guarantees from data.
To the best of our knowledge, ours is the first work to do this for general human-interpretable concepts, by exploiting ideas from the causal representation learning literature.

\section{Setup}

In this section, we provide a formal definition of concepts, which are high-level abstractions present in data. This allows us to develop a theoretical framework for associated data distributions and identifiability theory.
For the sake of intuition, we can think of the data as images of different objects and the color of the object as a concept.

\subsection{Generative model}
\label{sec: model}

We assume that the observed data $X$ lies in a space $\calX \subseteq \RR^{d_x}$ of dimension $d_x$ and has an underlying representation $X = f(Z)$ for latent variables $Z$ that lie in a latent concept space $\RR^{d_z}$ of dimension $d_z$.
In contrast to most prior works we do not take the viewpoint that $Z$ represents
the true underlying mechanism that generated the data. Instead we simply assume that the latent
representation has the geometric property that it maps certain regions of the observation space to
linear subspaces of the latent space (motivated by the foundation models literature). Then we investigate to what extent this assumption results in the
identifiability of such representations.
\begin{assumption}[Mixing function]\label{as:mixing}
    The non-linear $f$ is injective and differentiable.
\end{assumption}

Injectivity and differentiability are standard assumptions in representation learning.
However, we make no additional assumptions on $f$: The map from latent space to observation space can be arbitrarily non-linear. For the rest of this subsection, fix a data representation $f^{-1}$.

We now define concepts living in the latent space $\RR^{d_z}$.
Informally, we think of an (atomic) concept as a vector $a \in \RR^{d_z}$ such that for a given valuation $b \in \RR^n$, the set of all observations $X$ that satisfy this concept is given by $\{X = f(Z) | \ip{a}{Z} = b\}$. For instance, for an object in an image $X$, if $a \in \RR^{d_z}$ is the concept of red color, $b \in \RR$ could indicate the intensity; then all datapoints $X$ satisfying this concept, i.e., all images with an object that has color red with intensity $b$, can be characterized as $X = f(Z)$ where $Z$ satisfies $\ip{a}{Z} = b$.
For a 3D visualization, see \cref{fig: hyperplanes}
We make this intuition formal below.

\begin{figure*}[!t]
\captionsetup[subfigure]{oneside,margin={0.5cm,0cm}}
\begin{subfigure}{0.5\textwidth}
    \centering
\begin{tikzpicture}[
            > = stealth, %
            shorten > = 1pt, %
            auto,
            node distance = 3cm, %
            semithick, %
            label distance=1mm
        ]

        \tikzstyle{every state}=[
            draw = black,
            thick,
            fill = white,
            minimum size = 4mm,
            text width=4mm,
        ]
        \node[state] (X) {$X$};
        \node[state] (Z1)
        [above left of = X]
        {$Z_u$};
        \node[state] (Z2)
        [above right of = X]
        {$Z$};
        \path[->,draw,thick]
          (Z1) edge node[align = center, swap] {$f_u$} (X)
          (Z2) edge node[align = center] {$f$} (X);
        \path[->,draw,dashed]
          (Z1) edge node[align = center] {\footnotesize{Preserve concept linearity}} (Z2);
\end{tikzpicture}
\hspace{1em}\caption{$Z_u$ and $f_u$ are underlying and arbitrarily complicated, whereas (simpler) $Z, f$ are learned.}
\label{fig: goal}
\end{subfigure}
\begin{subfigure}{0.5\textwidth}
    \centering
    \hspace{5em}
    \includegraphics[scale=.4,trim={1cm .5cm 0cm 0cm}, clip]{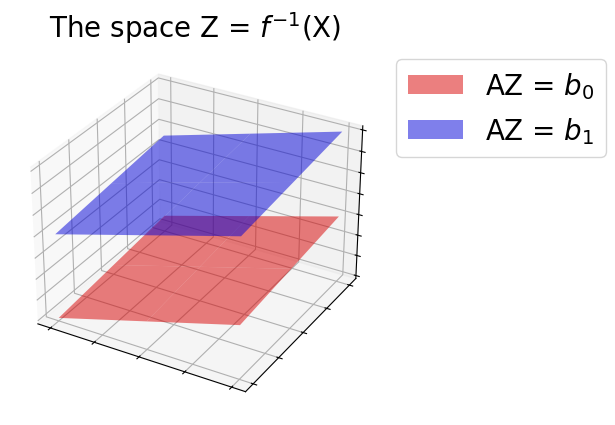}
    \caption{Concepts live in affine subspaces. The two subspaces in the figure correspond to the same concept but of different valuations.}
    \label{fig: hyperplanes}
\end{subfigure}
\caption{Illustration of our framework}
\label{fig:setup}
\vspace{-11pt}
\end{figure*}
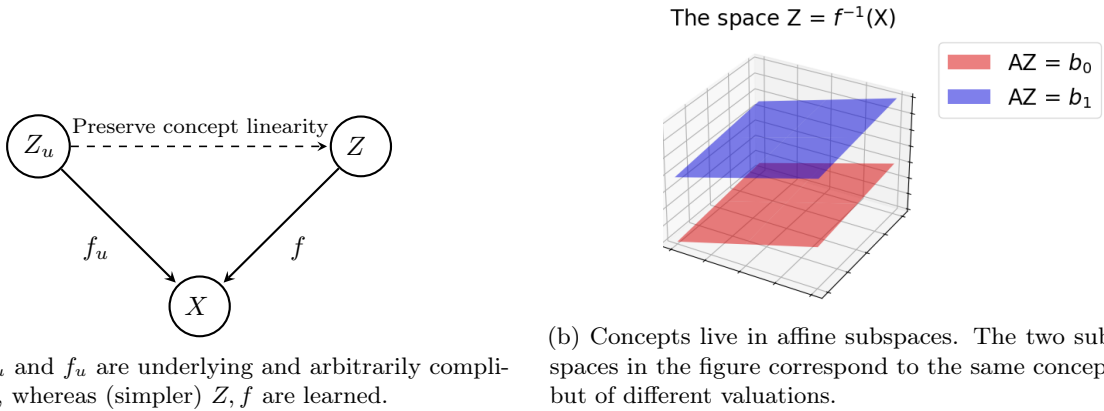

\begin{definition}[Concepts]\label{def:concept}
    A concept $C$ is a linear transformation $A: \RR^{d_z} \to \RR^{d_C}$.
    The dimension of the concept will be denoted by $\dim(C)=d_C$.
    A  valuation is a vector $b \in \RR^{d_C}$ and we say that  a datapoint $X$ satisfies the concept $C$ with valuation $b$ if the evaluation map satisfies $AZ = b$ where $Z = f^{-1}(X)$.
\end{definition}
In this work, we are  interested in learning a collection of $m$ concepts $C^1, \ldots, C^m$ from observed data.
By left multiplying by the pseudo-inverse $A^+$, we can equivalently assume $A$ is a projector matrix. However, the current definition
is more suitable for embeddings of real models.

When we talk of learning concepts $C$, we are in particular interested in learning
the evaluation map $Af^{-1}(x)$. This is a more modest objective than learning the entire
map $f$ which is the usual goal in identifiability theory. While the latter typically requires
stringent assumptions, in particular $\Omega(d_z)$
environments are necessary, our weaker identifiability results only need $O(d_C) \ll O(d_z)$ environments.
To simply our analysis, we can view each concept as being composed of one dimensional concepts, which warrants the following definition.

\begin{definition}[Atoms]
\label{def: atomic}
    We define an atom (short for atomic concept) to be any concept $C$ with $\dim(C) = 1$.
\end{definition}

Intuitively, atomic concepts are fundamental concepts that live in a space of co-dimension $1$ in latent space, and thus are equivalently defined by vectors $a\in\RR^{d_z}$. For the sake of intuition, we can think of red color, size of object, etc., as examples of atomic concepts.
Any generic concept is then composed of a collection of atomic concepts, e.g., the concept $C$ of all small dark red objects will correspond to $dim(C) = 2$ with row $1$ corresponding to the atomic concept of red color with large valuation (dark red objects) and row $2$ corresponding to the atomic concept of object size with low valuation (small objects).

\subsection{Data distributions}
\label{sec: data_distributions}

We now define the distributions of datasets over concepts. We will predominantly work with distributions of $Z$ over $\RR^{d_z}$, as the resulting distribution of $X = f(Z)$ over $\RR^{d_x}$ can be obtained via a simple change of variables.

To build intuition, consider the case where we first collect a base dataset with some underlying distribution and then collect concept datasets via filtering. For instance, we could first collect a set of images of all objects and then, to collect a dataset of dark red colored objects, we filter them to only keep images of dark red colored objects. We call the former the \textit{base distribution} and the latter the \textit{concept conditional distribution} corresponding to our concept.

Fix a nonlinearity $f$.
We assume that the base data distribution is the distribution of $X = f(Z)$ with $Z \sim p$, where $p$ is the underlying distribution on $\RR^{d_z}$. In what follows, we will abuse notation and use $p$ for both the distribution and the corresponding probability density which we assume exists.
We make no further assumptions on $p$ since we do not wish to model the collection of real-life datasets that have been collected from nature and which could be very arbitrary.

We now define the concept conditional distribution, which is a distribution over $X$ that is induced by noisy observations of a particular concept at a particular valuation. Formally, assume we want to condition on some atomic concept $a\in \RR^{d_z}$ with valuation $b$.
It is reasonable to assume that this conditioning is a noisy operation. For instance, human beings are great at distilling concepts from noised images, e.g., they recognize
 cars in a misty environment.
We formalize this by assuming that
 data collection is based on a noisy estimate
$\alt{b}=\langle a, z\rangle + \eps$ where $\eps$ is independent of $z$ and its density is
a symmetric distribution with density $q(\eps)$.
Then we consider the concept conditional distribution
\begin{align}
\begin{split}
    p_C(z)=p(z|\alt{b} = b)&\propto
    p(\alt{b}=b|z) p(z)
    \\
    &=q(b-\langle a, z\rangle )p(z)
    \end{split}
\end{align}
where we used Bayes theorem in the last step.
 This definition directly extends to
 higher dimensional concepts which are concisely defined as follows.

\begin{definition}[Concept conditional distribution]\label{def:conditional}
For a concept $C$ with associated linear map $A$ and an arbitrary valuation $b \in \RR^{dim(C)}$, we define the concept conditional distribution to be the set of observations $X$ respecting this concept, which is defined as the resampled distribution of $X = f(Z)$ where $Z \sim p_C$ with
\begin{align}\label{eq:concept_measure}
p_C(Z) \propto p(Z)
 \prod_{k \le dim(C)} q((AZ - b)_k)
 \end{align}
\end{definition}
This is by no means the only possible definition, and we present feasible alternate definitions in \cref{app:alt_defns_concept_distributions}.
We remark that our formulation is related to the iVAE setting \citep{khemakhem2020variational}
and the auxiliary variable setting for identifiable ICA in \citet{hyvarinen2019nonlinear} and we discuss the relation later.
Note that the majority of recent identifiability results relied on interventional data while we only consider conditional information here.

We are ready to define our main problem of interest.

\begin{problem}
    We are given an observational dataset $X^0=f(Z^0)$ corresponding to the latent base distribution $p$ along with datasets $X^1, \ldots, X^m$ corresponding to concept conditional datasets for different concepts $C^1, \ldots, C^m$ and corresponding valuations $b^1, \ldots, b^m$
    over the same latent space $\RR^{d_z}$ with  the same mixing  $f$. Under what conditions (and up to which symmetries) can we learn the concepts $C^1, \ldots, C^m$, which includes the linear maps $A^1, \ldots, A^m$, and the concept valuations $A^ef^{-1}(x)$?
\end{problem}

Toward this end, a fundamental question is whether this problem is even possible, i.e., whether it is well-defined. This ties to the question of identifiability of such parameters of interest. Therefore, we make the following definition.
Informally, for the setting above, we say that the concepts $(C^1, A^1), \ldots, (C^m, A^m)$ with associated nonlinearity $f$ are identifiable (and thus learnable) if for any other collection of different parameters that fit the data, they are linearly related to the true parameters.

\begin{definition}[Identifiability]\label{def:identifiability}
Given datasets $X^0$, $X^1, \ldots, X^m$ corresponding to the observational distribution and $m$ concepts $C^1, \ldots, C^m$ with underlying distribution $p$ on $\RR^{d_z}$, nonlinearity $f$, linear maps $A^1, \ldots, A^m$ and valuations $b^1, \ldots, b^m$, we say the concepts are identifiable if the following holds: Consider any different collection of parameters $\alt{f}, \alt{d_z}, \alt{p}$, concepts $(\alt{C^1}, \alt{A^1}), \ldots, (\alt{C^m}, \alt{A^m})$ and valuations $\alt{b^1}, \ldots, \alt{b^m}$ that also generate the same observations $X^0, X^1, \ldots, X^m$.
Then there exists a shift $w \in \RR^{d_z}$, permutation matrices $P^e$ and invertible diagonal matrices $\Lambda^e$ such that for all $e$ and $x$,
\begin{align}\label{eq:ident1}
\alt{A}^e \alt{f}^{-1}(x) =   \Lambda^e P^eA^e(f^{-1}(x) + w),
\end{align}
i.e., we can evaluate the concept evaluations on the data up to linear reparametrizations.
Moreover, there exists a linear map $T: \RR^{\alt{d_z}} \to \RR^{d_z}$ such that
the concepts and their evaluations are related as follows
\begin{align}\label{eq:ident2}
    \alt{A}^e= P^eA^e T^{-1}, \quad \alt{b}^e=\Lambda^e P^e(b^e-A^ew).
\end{align}
\end{definition}
The main message of this definition is that identifiability implies we can identify the nonlinear map $f^{-1}$ within the span of the subspace of the concepts of interest, and therefore we can recover the concepts of interest from our data. That is, if certain concepts are identifiable, then we will be able to learn these concept representations up to linearity, even if they can be highly nonlinear functions of our data. Such concept discovery is useful because they can then be used for further downstream tasks such as controllable generative modeling.

We emphasize that in contrast to previous work we are not aiming to identify $f$
completely.
Let us now clarify why no stronger identifiability results can be expected.
Firstly, we cannot hope to resolve the linear transformation ambiguity because the latent space is not directly observed.
In other words, a concept evaluation can be defined either as $\ip{a}{Z}$ or as $\ip{Ta}{T^{-\top}Z}$
for an invertible linear map $T$.
However, for the purposes of downstream tasks, this is fine since the learned concepts will still be the same in principle.
Secondly, we cannot expect to recover $f^{-1}$ outside the span of the concepts because we do not manipulate the linear spaces outside the span therefore we do not learn this information from our observed data so this is also tight.
The permutation matrix captures the fact that the ordering of the concepts does not matter.
Therefore, this definition captures the most general identifiability guarantee that we can hope for in our setting and furthermore, this suffices for downstream tasks such as controllable data generation.

Because we will only be interested in recovering the set of concepts up to linear transformations, without loss of generality, we will fix the base collection of atomic concepts.
That is, we assume that each concept $C^e$ corresponds to a linear map $A^e$ whose rows are a subset of $\calC$, where $\calC = \{a_1, \ldots, a_n\}$ is a set of atomic concepts  that we wish to learn.
Moreover, we assume that they are linearly  independent, since we want them to encode distinct concepts.

\begin{assumption}\label{as: lin_ind_atomic_concepts}
    There exists a set of atomic concepts $\calC = \{a_1, \ldots, a_n\}$ of linearly independent vectors such that each concept $C^e$ under consideration contains as rows a subset $S^e$ of $\calC$
 and all concepts in  $\calC$ appear, i.e., $\bigcup_e S^e=[n]$.
\end{assumption}
\begin{remk}
    Note that identifiability as defined in Definition~\ref{def:identifiability}
    implies that the atoms can be identified in the sense that there is a permutation $\pi \in S_n$
   and $\lambda_i\neq 0$ such that for $T$ as in Definition~\ref{def:identifiability} and some $\lambda_i$
   \begin{align}\label{eq:ident_atoms1}
    \alt{a}_{\pi(i)}^\top  &=  a_i^\top T^{-1}\\
    \label{eq:ident_atoms2}
        \langle \alt{a}_{\pi(i)},\alt{f}^{-1}(x)\rangle &=
      \lambda_i\left(  \langle a_i, f^{-1}(x)\rangle + \langle a_i,w\rangle\right),
    \end{align}
    i.e., we can evaluate the valuations of the atomic concepts up to linear reparametrization.
\end{remk}
\section{Main Result}

In this section, we will be interested to identify a class of distinct concepts from data. As discussed earlier, we wish to recover them as per our precise definition of identifiability, which as we saw is the best possible.

The punchline is that when we have rich datasets, i.e., sufficiently rich concept conditional datasets, then we can recover the concepts. Importantly, we only require a number of datasets that depends only on the number of atoms $n$ we wish to learn (in fact, $O(n)$ datasets), and not on the underlying latent dimension $d_z$ of the true generative process. This is a significant departure from most works on causal representation learning, since the true underlying generative process could have $d_z = 1000$, say, whereas we may be interested to learn only $n = 5$ concepts, say. In this case, causal representation learning necessitates at least $\sim 1000$ datasets, whereas we show that $\sim n+2 = 7$ datasets are enough
if we only want to learn the $n$ atomic concepts.
We will explain the connection to causal representation learning in \cref{app:crl}.
Let us now discuss our main assumptions.
\begin{assumption}\label{as:gaussianity}
    We choose the distribution $q$ to be Gaussian with mean $0$ and variance $\sigma^2$ for some
    $\sigma^2>0$.
\end{assumption}

We relate the concepts $C^e$ to the atoms. Recall that we defined the index sets $S^e=\{i \in [n]: a_i\in \calC \text{ is a row of $A^e$}\}$ of atomic concepts in environment $e$.
We define the  environment-concept matrix $M\in \RR^{m\times n}$ indexed by environments and
atoms by
\begin{align}\label{eq:env_conc}
    M_{ei} = \begin{cases}
        \frac{1}{\sigma^2} &\text{if $i\in S^e$}
        \\
        0&\text{otherwise.}
    \end{cases}
\end{align}
Similarly, we consider the  environment-valuation matrix $B\in \RR^{m\times n}$ given by
\begin{align}\label{eq:env_valuation}
    B_{ei} = \begin{cases}
        \frac{b^e_{k}}{\sigma^2} &\text{if $i\in S^e$ and row $k$ of $A^e$ is $a_i$,}
        \\
        0&\text{otherwise.}
    \end{cases}
\end{align}
We now state our assumptions on the diversity of the concept conditional distributions that will ensure identifiability.

\begin{assumption}[Environment diversity I]\label{ass:div}
The environment-concept matrix $M\in \RR^{m\times n}$ has rank $n$ and there is a vector $v\in \RR^m$
such that $v^\top M=0$ and all entries of $v^\top B$ are non-zero ($B$ denotes that environment-valuation matrix).
\end{assumption}
We remark that this assumption can only hold for  $m\geq n+1$
and indeed
is satisfied under mild assumptions on the environments if $m=n+1$, as the following lemma shows.
\begin{lemma}\label{le:generic_assumption}
    Assumption~\ref{ass:div} is satisfied almost-surely if there are $n+1$ concept conditional distributions
    such that every $n$ rows of  the environment-concept matrix are linearly independent
    and the $b^e$ are drawn independently according to a continuous distribution.
\end{lemma}
We also assume one additional diversity condition.
\begin{assumption}[Environment diversity II]\label{ass:div2}
For every pair of atoms $a_i$ and $a_j$ with $i\neq j$ there is an environment $e$ such that $i\in S^e$
and
$j\notin S^e$.
\end{assumption}

We remark that these are the only assumptions about the sets $S^e$. In particular, we do not need to know the sets $S^e$. In the proof, we will extract these sets based on a the signatures they leave on the datasets.
We can now state our main result.

\begin{theorem}
\label{thm: main}
        Suppose we are given $m$ context conditional datasets $X^1,\ldots, X^m$ and the observational dataset $X^0$ such that Assumptions \ref{as:mixing}, \ref{as: lin_ind_atomic_concepts}, \ref{as:gaussianity}, \ref{ass:div}, and \ref{ass:div2} hold. Then the concepts are identifiable as per \cref{def:identifiability}.
\end{theorem}
\begin{remk}
    Assumption~\ref{ass:div} can only be satisfied for $m\geq n+1$, i.e., the result requires at least  $n+2$ environments. On the other hand, Lemma~\ref{le:generic_assumption} assures
    that $n+2$ environments are typically sufficient. We expect that the  result could be slightly improved by showing identifiability for $n+1$ environments under suitable assumptions. However, this would probably require more advanced techniques from algebraic statistics \citep{drton2009} compared to the necessary techniques for our result.
\end{remk}

As mentioned before, our setting somewhat resembles the iVAE setting in \citet{khemakhem2020variational}
and therefore, their proof techniques can also be applied, with several modifications, to derive identifiability results in our setting (however our formulation and application are very different).
We notably remark that this approach will require more environments because their main assumption is that
the matrix $\Lambda = (M, B)\in \RR^{m\times 2n}$ has rank $2n$ so that $2n+1$ environments are necessary. Moreover, this rank condition is much stronger than Assumption~\ref{ass:div}.
For completeness and as a warm-up we prove this result in Appendix~\ref{sec: proofs}.

The full proof of \cref{thm: main} can be  found in Appendix~\ref{sec: proofs}.
Although the proof is fairly involved, we present the core idea here.
Given two sets of parameters that fit the data, we show they're linearly related. To do this, we first equate the log-odds, which by assumption will be quadratic polynomials. Taking derivatives, we will be able to obtain linear equations with coefficients depending on our environment matrices. By a careful linear algebraic analysis, we can conclude identifiability. The diversity assumptions ensure invertibility of various matrices that show up, which lets us solve the equations.

\section{Experiments}

We present two slates of experiments to supplement our framework. In \cref{sec: synthetic}, we validate our
results on synthetic data, via an end-to-end contrastive learning algorithm for concept learning. We additionally propose a rejection sampling algorithm to sample from concept conditional distributions in \cref{sec: sampling}.
In \cref{sec: realworld}, we focus on real-world settings especially on large language models (LLMs). We show how our techniques can be used to align pre-trained LLMs towards abstract concepts such as truthfulness.

\subsection{End-to-end Contrastive learning algorithm}
\label{sec: synthetic}

In this section, we will validate our insights on synthetically generated data in order to verify our results on identifiability.
We present an end-to-end framework based on contrastive learning (similar to \cite{buchholz2023learning}) to learn the nonlinearity as well as concepts from data.
The model architecture is designed based on our concept conditional distribution parametrization.

The core idea is to utilize contrastive learning as follows.
For each concept conditional distribution $X^e$, we train a neural network to distinguish concept samples $x \sim X^e$ from base samples $x \sim X^0$.
First, we compute the true log-odds for this classification problem.

\begin{restatable}{lemma}{logodds}
\label{lem: logodds}
For any concept index $e$, there exist some constants $c_e$  such that
\begin{align*}
   &\ln(p^e(Z)) - \ln(p(Z))
  \\
  &\qquad= \sum_{i=1}^n\left( -\frac{1}{2} M_{ei} \langle a_i, Z^e\rangle^2 + B_{ei} \langle a_i,Z^e\rangle\right) + c_e
\end{align*}
where $M, B$ are the environment-concept matrix and the environment-valuation matrix defined in \eqref{eq:env_conc} and \eqref{eq:env_valuation}.
\end{restatable}

The proof is in \cref{sec: contrastive}.
Now, let's try to learn the $n$ atomic concepts up to linearity. To do this, we build a neural architecture for this classification problem with the final layer mimicking the log-odds expression above, which can then be trained end-to-end. Because of the careful parametrization of the last layer, this will encourage the model to learn the representations as guaranteed by our results.
The details are deferred to \cref{sec: contrastive}.

\paragraph{Sampling from concept conditional distributions}
A common task in controllable generative modeling is being able to generate data from a known concept.
Note that this is not straightforward in our setting because the normalization term in \cref{eq:concept_measure} is not efficiently computable. To do this efficiently, we also outline a simple algorithm (Algorithm \ref{algo: sampling} in \cref{sec: sampling}) to sample from the concept conditional distribution for a known concept.
Our proposed algorithm is based on rejection sampling and the algorithm as well as the complexity analysis are deferred to \cref{sec: sampling}.

\paragraph{Synthetic experiments}
We test our proposed method on synthetic data as follows. We sample the base distribution from a Gaussian Mixture model and experiment with both linear and nonlinear mixing functions. The parameters are all chosen randomly.
The number of concepts $n$ is intentionally chosen to be less than the ground truth dimension $d_z$ and the number of concepts is $m = n + 1$ as per our theory.
Additional details are deferred to \cref{sec: synth_setup}.

\begin{table}[!h]
  \centering
\caption{Linear identifiability when number of concepts $n$ is less than underlying latent dimension $d_z$, averaged over $5$ seeds}
\label{tab: synth}
{\footnotesize
\begin{tabular}{llll|ll}
\toprule
Type of mixing $f$ & $n$ & $d_z$ & Obs dim $d_x$ & $R^2$$\uparrow$ & MCC$\uparrow$\\
\midrule
Linear & 2 & 3 & 4 & 0.98 & 0.98 \\
Nonlinear & 2 & 3 & 4 & 0.94 & 0.96 \\
\midrule
Linear & 3 & 4 & 6 & 0.99 & 0.86 \\
Nonlinear & 3 & 4 & 6 & 0.97 & 0.92 \\
\midrule
Linear & 4 & 8 & 10 & 0.97 & 0.87 \\
Nonlinear & 4 & 8 & 10 & 0.94 & 0.87 \\
\bottomrule
\end{tabular}}
\end{table}
In \cref{tab: synth}, we show the metrics achieved by our contrastive learning algorithm on this data, where we report the $R^2$ and Mean Correlation Coefficient (MCC) metrics \cite{khemakhem2020variational, khemakhem2020ice} of the recovered latents with respect to the ground truth concept valuations.
There are no baselines since we are in a novel setting where standard causal representation learning baselines do not easily apply.
However, our metrics are comparable to what's usually reported in such highly nonlinear settings \cite{willetts2021don, buchholz2023learning}.

We remark that variations of the contrastive method can be designed for harder synthetic settings and different problems related to concept discovery.
However, we will move onto real-life data experiments next.

\subsection{Alignment of Large Language Models}
\label{sec: realworld}

In this section, we show an application of our framework to alignment of Large Language Models (LLMs).
In particular, we exploit our ideas to improve the Inference-Time Intervention technique \cite{li2023inference} to promote LLMs to be more truthful.
In contrast to previous sections, we will focus on exactly one concept -- \textit{truthfulness} with two relevant valuations -- \textit{false} and \textit{true}.
We concretize this by assuming this concept is linear and the valuations correspond to real numbers (as evidenced by \citet{li2023inference}, see also \cref{sec: related_work} for more evidence).
The downstream task is to make pre-trained LLMs answer truthfully, i.e., change the valuation of this concept from \textit{false} to \textit{true}, without affecting any other orthogonal concepts.

Thus, in contrast to the previous section where we trained an end-to-end model to learn both the non-linearities as well as the concepts, we will now assume that the non-linearity has already been learned up to a linear transformation (by large-scale training of LLMs). This aligns with our theoretical insights because the training data for powerful LLMs are diverse, so they essentially satisfy our core assumptions (see also the related work \cite{gupta2023context} that proposes that context is environment in LLM training). Therefore, we simply focus on the downstream tasks, which in this section is LLM alignment.
The difficulty, of course, is that we do not know the concept matrix or the valuations.
A detailed introduction to large language models (LLMs) and the Inference-Time Intervention (ITI) technique is deferred to \cref{sec: iti}. We present a simplified summary here.

Given a context $x$, a GPT-style LLM outputs a representation $f(x)$, which is then decoded into probabilities of the next token.
ITI is an activation patching technique that controls the behavior of LLMs during inference time, in particular it promotes truthfulness.
ITI proposes to modify the forward pass (during inference time) to $f(x) + \al \sig \cdot \eta$ where $\eta$ is a trained steering vector, $\sig$ is the standard deviation of the representations in this direction and $\al$ is a tunable hyperparameter.\footnote{Note that the simplified version is stated here. In ITI, $\eta$ is not actually a constant vector because it's an autoregressive model. However a constant vector is added to each layer during inference time (which is sufficient for relaying intuition in this section). See \cref{sec: iti} for details.}

Here, the steering vectors are intended to steer the LLM towards truthful responses. To train them, the work \cite{li2023inference} uses a set of counterfactual pairs $(c_i^F, c_i^T)$ of false and true reponses to the same questions (for instance from the TruthfulQA dataset \cite{lin2021truthfulqa}), compute the difference of their representations $f(c_i^T) - f(c_i^F)$ and the normalized mean difference is used as the steering vector $\eta$.

First, we use our framework to mechanistically reason about how ITI modifies the internal representations and illustrate why the mean is a good choice.
We present an intuitive summary here, postponing details to \cref{sec: iti}. If we consider the concept conditional distributions of the true and false valuations of our truthfulness concept (assume for simplicity it's an atomic concept $a$ with valuations $b^T, b^F$), then they concentrate in their respective hyperplanes $\langle a, f(c_i^T)\rangle = b^T$ and $\langle a, f(c^F_i)\rangle = b^F$.
Now, consider the steering vector $\eta$, the normalized mean of the $f(c_i^T) - f(c_i^F)$.
The crucial insight is that if we have a diverse set of counterfactual pairs, then in the hyperplane orthogonal to $a$, the fluctuations of $f(c_i^T)-f(c_i^F)$ get averaged and should disappear when taking the mean. Therefore, $\eta$ will be parallel to $a$, which is the optimal steering vector. This intuition can be generalized to non-atomic concepts, as we will describe in  Appendix~\ref{sec: steering_vector}.
In summary, our framework enables progress on mechanistic interpretability \cite{olah2022}, an emerging field which aims to understand LLM mechanisms.

Next, we propose a modification of this technique as follows.
First, we use our insights to arrive at the intuition that any weighted mean of the counterfactual differences should roughly have the same steering behavior. And moreover, we choose the weights dynamically, with weights being similarities of the test sentence to the sentences from the training set. This enables similar sentences in the training set to have a higher vote in the choice of steering vector.
With this new idea, we use as steering vector
\begin{align*}
\eta(x) &= \sum_{i}\ip{\lda(x)}{\lda(c_i^F)}(f(c_i^T) - f(c_i^F))
\end{align*}
where $\lda$ is a sentence embedding (such as SBERT \cite{reimers-2019-sentence-bert}).
We detail this further in \cref{sec: steering_vector}.
Finally, we also show how to implement this efficiently using \textit{steering matrices} instead of steering vectors.

\begin{table}[!ht]
  \centering
\caption{Comparison of steering vectors for LLM alignment}
\label{tab: llama}
{\footnotesize
\begin{tabular}{lllll}
\toprule
Technique & $\alpha$ & Acc (\%) & CE loss & KL div. \\
\midrule
Baseline & - & 25.7 & 2.16  & 0.0 \\
Random direction & 20 & 25.8 & 2.19  & 0.02 \\
CCS direction & 5 & 26.2  & 2.21  & 0.06  \\
ITI: Probe weight dir. & 15 & 27.0 & 2.21  & 0.06  \\
ITI: Mass mean shift & 20 & 28.8 & 2.41  & 0.27 \\
\midrule
Steering matrices (ours) & 15 & \textbf{29.5} & 2.61 & 0.41\\
\bottomrule
\end{tabular}}
\end{table}

In \cref{tab: llama}, we show the results of our experiments with steering matrices.
We use the open-source large language model LLaMA \cite{touvron2023llama} with 7 billion parameters and the sentence transformer SBERT \cite{reimers-2019-sentence-bert} for the sentence embedding.
We report the accuracy of the multiple-choice track of TruthfulQA \cite{li2023inference} over 3 random seeds and also the Cross-Entropy Loss and KL divergence of the model pre- and post-intervention. Higher accuracy is better and lower CE loss, and KL divergence indicate that the original model has not been significantly modified. Here, the baselines are the unmodified model, random direction intervention, Contrast-Consistent Search (CCS) direction \cite{burns2022discovering} and two different direction choices using vanilla ITI; and 2-fold cross validation is used.
Additional technical details are deferred to \cref{sec: iti}.

We see that the multiple-choice accuracy improved, showcasing the potential of our steering matrices technique which is novel in the field of LLM alignment to the best of our knowledge.
Note that this is meant to be a proof of concept and not meant to be a comprehensive study of this specific technique.
For exploratory purposes, we outline potential modifications to our technique in \cref{sec: iti} which could potentially improve the performance, both in terms of accuracy as well as in terms of invasiveness.
These form an exciting direction for a more comprehensive study of our proposed ideas, which we leave for future work.

\section{Conclusion}

In this work, we exploited ideas from both the causal representation learning and the foundations model literature to study the problem of learning human-interpretable concepts from data. We proposed a formal definition of concepts and studied under what conditions they can be provably recovered, suggesting what representations foundation models learn. Finally, synthetic experiments and large language model alignment experiments showcase the utility of our ideas.

While we have made initial progress in unifying these disparate fields, this direction holds a lot of promise.
We believe there's a lot of upside in utilizing ideas from identifiable representation learning in order to theoretically analyze what representations foundation models learn, as it will help us understand, explain, and improve their capabilities.

\paragraph{Acknowledgments}

We acknowledge the support of AFRL and DARPA via FA8750-23-2-1015, ONR via N00014-23-1-2368, NSF via IIS-1909816, IIS-1955532, IIS-1956330, and NIH R01GM140467. We also acknowledge the support of the T\"ubingen AI Center and the Robert H. Topel Faculty Research Fund at the University of Chicago Booth School of Business.

\bibliography{main}

\begin{thebibliography}{116}
\providecommand{\natexlab}[1]{#1}
\providecommand{\url}[1]{\texttt{#1}}
\expandafter\ifx\csname urlstyle\endcsname\relax
  \providecommand{\doi}[1]{doi: #1}\else
  \providecommand{\doi}{doi: \begingroup \urlstyle{rm}\Url}\fi

\bibitem[Ahuja et~al.(2023{\natexlab{a}})Ahuja, Mahajan, Wang, and
  Bengio]{ahuja2022interventional}
K.~Ahuja, D.~Mahajan, Y.~Wang, and Y.~Bengio.
\newblock Interventional causal representation learning.
\newblock In \emph{Proceedings of the 40th International Conference on Machine
  Learning}, ICML'23. JMLR.org, 2023{\natexlab{a}}.

\bibitem[Ahuja et~al.(2023{\natexlab{b}})Ahuja, Mansouri, and
  Wang]{ahuja2023multi}
K.~Ahuja, A.~Mansouri, and Y.~Wang.
\newblock Multi-domain causal representation learning via weak distributional
  invariances.
\newblock \emph{arXiv preprint arXiv:2310.02854}, 2023{\natexlab{b}}.

\bibitem[Allen and Hospedales(2019)]{allen2019analogies}
C.~Allen and T.~Hospedales.
\newblock Analogies explained: Towards understanding word embeddings.
\newblock In \emph{International Conference on Machine Learning}, pages
  223--231. PMLR, 2019.

\bibitem[Arora et~al.(2016)Arora, Li, Liang, Ma, and Risteski]{arora2016latent}
S.~Arora, Y.~Li, Y.~Liang, T.~Ma, and A.~Risteski.
\newblock A latent variable model approach to pmi-based word embeddings.
\newblock \emph{Transactions of the Association for Computational Linguistics},
  4:\penalty0 385--399, 2016.

\bibitem[Bai et~al.(2022{\natexlab{a}})Bai, Jones, Ndousse, Askell, Chen,
  DasSarma, Drain, Fort, Ganguli, Henighan, et~al.]{bai2022training}
Y.~Bai, A.~Jones, K.~Ndousse, A.~Askell, A.~Chen, N.~DasSarma, D.~Drain,
  S.~Fort, D.~Ganguli, T.~Henighan, et~al.
\newblock Training a helpful and harmless assistant with reinforcement learning
  from human feedback.
\newblock \emph{arXiv preprint arXiv:2204.05862}, 2022{\natexlab{a}}.

\bibitem[Bai et~al.(2022{\natexlab{b}})Bai, Kadavath, Kundu, Askell, Kernion,
  Jones, Chen, Goldie, Mirhoseini, McKinnon, et~al.]{bai2022constitutional}
Y.~Bai, S.~Kadavath, S.~Kundu, A.~Askell, J.~Kernion, A.~Jones, A.~Chen,
  A.~Goldie, A.~Mirhoseini, C.~McKinnon, et~al.
\newblock Constitutional ai: Harmlessness from ai feedback.
\newblock \emph{arXiv preprint arXiv:2212.08073}, 2022{\natexlab{b}}.

\bibitem[Bau et~al.(2017)Bau, Zhou, Khosla, Oliva, and
  Torralba]{bau2017network}
D.~Bau, B.~Zhou, A.~Khosla, A.~Oliva, and A.~Torralba.
\newblock Network dissection: Quantifying interpretability of deep visual
  representations.
\newblock In \emph{Proceedings of the IEEE conference on computer vision and
  pattern recognition}, pages 6541--6549, 2017.

\bibitem[Bengio et~al.(2013)Bengio, Courville, and
  Vincent]{bengio2013representation}
Y.~Bengio, A.~Courville, and P.~Vincent.
\newblock Representation learning: A review and new perspectives.
\newblock \emph{IEEE transactions on pattern analysis and machine
  intelligence}, 35\penalty0 (8):\penalty0 1798--1828, 2013.

\bibitem[Brehmer et~al.(2022)Brehmer, De~Haan, Lippe, and
  Cohen]{brehmer2022weakly}
J.~Brehmer, P.~De~Haan, P.~Lippe, and T.~S. Cohen.
\newblock Weakly supervised causal representation learning.
\newblock \emph{Advances in Neural Information Processing Systems},
  35:\penalty0 38319--38331, 2022.

\bibitem[Buchholz et~al.(2022)Buchholz, Besserve, and
  Sch{\"o}lkopf]{buchholz2022function}
S.~Buchholz, M.~Besserve, and B.~Sch{\"o}lkopf.
\newblock Function classes for identifiable nonlinear independent component
  analysis.
\newblock In A.~H. Oh, A.~Agarwal, D.~Belgrave, and K.~Cho, editors,
  \emph{Advances in Neural Information Processing Systems}, 2022.
\newblock URL \url{https://openreview.net/forum?id=DpKaP-PY8bK}.

\bibitem[Buchholz et~al.(2023)Buchholz, Rajendran, Rosenfeld, Aragam,
  Sch{\"o}lkopf, and Ravikumar]{buchholz2023learning}
S.~Buchholz, G.~Rajendran, E.~Rosenfeld, B.~Aragam, B.~Sch{\"o}lkopf, and
  P.~Ravikumar.
\newblock Learning linear causal representations from interventions under
  general nonlinear mixing.
\newblock \emph{arXiv preprint arXiv:2306.02235}, 2023.

\bibitem[Burns et~al.(2022)Burns, Ye, Klein, and
  Steinhardt]{burns2022discovering}
C.~Burns, H.~Ye, D.~Klein, and J.~Steinhardt.
\newblock Discovering latent knowledge in language models without supervision.
\newblock \emph{arXiv preprint arXiv:2212.03827}, 2022.

\bibitem[Carvalho et~al.(2019)Carvalho, Pereira, and
  Cardoso]{carvalho2019machine}
D.~V. Carvalho, E.~M. Pereira, and J.~S. Cardoso.
\newblock Machine learning interpretability: A survey on methods and metrics.
\newblock \emph{Electronics}, 8\penalty0 (8):\penalty0 832, 2019.

\bibitem[Chiang et~al.(2023)Chiang, Li, Lin, Sheng, Wu, Zhang, Zheng, Zhuang,
  Zhuang, Gonzalez, Stoica, and Xing]{vicuna2023}
W.-L. Chiang, Z.~Li, Z.~Lin, Y.~Sheng, Z.~Wu, H.~Zhang, L.~Zheng, S.~Zhuang,
  Y.~Zhuang, J.~E. Gonzalez, I.~Stoica, and E.~P. Xing.
\newblock Vicuna: An open-source chatbot impressing gpt-4 with 90\%* chatgpt
  quality, March 2023.
\newblock URL \url{https://lmsys.org/blog/2023-03-30-vicuna/}.

\bibitem[Comon(1994)]{comon1994independent}
P.~Comon.
\newblock Independent component analysis, a new concept?
\newblock \emph{Signal processing}, 36\penalty0 (3):\penalty0 287--314, 1994.

\bibitem[Conneau et~al.(2018)Conneau, Kruszewski, Lample, Barrault, and
  Baroni]{conneau2018you}
A.~Conneau, G.~Kruszewski, G.~Lample, L.~Barrault, and M.~Baroni.
\newblock What you can cram into a single vector: Probing sentence embeddings
  for linguistic properties.
\newblock \emph{arXiv preprint arXiv:1805.01070}, 2018.

\bibitem[Cui et~al.(2022)Cui, Huang, Wang, and Wang]{cuiaggnce}
J.~Cui, W.~Huang, Y.~Wang, and Y.~Wang.
\newblock Aggnce: Asymptotically identifiable contrastive learning.
\newblock In \emph{NeurIPS Workshop}, 2022.

\bibitem[D'Amour et~al.(2022)D'Amour, Heller, Moldovan, Adlam, Alipanahi,
  Beutel, Chen, Deaton, Eisenstein, Hoffman, et~al.]{d2022underspecification}
A.~D'Amour, K.~Heller, D.~Moldovan, B.~Adlam, B.~Alipanahi, A.~Beutel, C.~Chen,
  J.~Deaton, J.~Eisenstein, M.~D. Hoffman, et~al.
\newblock Underspecification presents challenges for credibility in modern
  machine learning.
\newblock \emph{The Journal of Machine Learning Research}, 23\penalty0
  (1):\penalty0 10237--10297, 2022.

\bibitem[Dilokthanakul et~al.(2016)Dilokthanakul, Mediano, Garnelo, Lee,
  Salimbeni, Arulkumaran, and Shanahan]{dilokthanakul2016deep}
N.~Dilokthanakul, P.~A. Mediano, M.~Garnelo, M.~C. Lee, H.~Salimbeni,
  K.~Arulkumaran, and M.~Shanahan.
\newblock Deep unsupervised clustering with gaussian mixture variational
  autoencoders.
\newblock \emph{arXiv preprint arXiv:1611.02648}, 2016.

\bibitem[Drton et~al.(2009)Drton, Sturmfels, and Sullivant]{drton2009}
M.~Drton, B.~Sturmfels, and S.~Sullivant.
\newblock \emph{Lectures on Algebraic Statistics}, volume~39 of
  \emph{Oberwolfach Seminars}.
\newblock Springer, 2009.
\newblock \doi{10.1007/978-3-7643-8905-5}.

\bibitem[Elhage et~al.(2021)Elhage, Nanda, Olsson, Henighan, Joseph, Mann,
  Askell, Bai, Chen, Conerly, et~al.]{elhage2021mathematical}
N.~Elhage, N.~Nanda, C.~Olsson, T.~Henighan, N.~Joseph, B.~Mann, A.~Askell,
  Y.~Bai, A.~Chen, T.~Conerly, et~al.
\newblock A mathematical framework for transformer circuits.
\newblock \emph{Transformer Circuits Thread}, 1, 2021.

\bibitem[Elhage et~al.(2022)Elhage, Hume, Olsson, Schiefer, Henighan, Kravec,
  Hatfield-Dodds, Lasenby, Drain, Chen, Grosse, McCandlish, Kaplan, Amodei,
  Wattenberg, and Olah]{elhage2022superposition}
N.~Elhage, T.~Hume, C.~Olsson, N.~Schiefer, T.~Henighan, S.~Kravec,
  Z.~Hatfield-Dodds, R.~Lasenby, D.~Drain, C.~Chen, R.~Grosse, S.~McCandlish,
  J.~Kaplan, D.~Amodei, M.~Wattenberg, and C.~Olah.
\newblock Toy models of superposition.
\newblock \emph{Transformer Circuits Thread}, 2022.
\newblock https://transformer-circuits.pub/2022/toy{\_}model/index.html.

\bibitem[Engel et~al.(2017)Engel, Hoffman, and Roberts]{engel2017latent}
J.~Engel, M.~Hoffman, and A.~Roberts.
\newblock Latent constraints: Learning to generate conditionally from
  unconditional generative models.
\newblock \emph{arXiv preprint arXiv:1711.05772}, 2017.

\bibitem[Ethayarajh et~al.(2018)Ethayarajh, Duvenaud, and
  Hirst]{ethayarajh2018towards}
K.~Ethayarajh, D.~Duvenaud, and G.~Hirst.
\newblock Towards understanding linear word analogies.
\newblock \emph{arXiv preprint arXiv:1810.04882}, 2018.

\bibitem[Falck et~al.(2021)Falck, Zhang, Willetts, Nicholson, Yau, and
  Holmes]{falck2021multi}
F.~Falck, H.~Zhang, M.~Willetts, G.~Nicholson, C.~Yau, and C.~C. Holmes.
\newblock Multi-facet clustering variational autoencoders.
\newblock \emph{Advances in Neural Information Processing Systems}, 34, 2021.

\bibitem[Ganguli et~al.(2022)Ganguli, Lovitt, Kernion, Askell, Bai, Kadavath,
  Mann, Perez, Schiefer, Ndousse, et~al.]{ganguli2022red}
D.~Ganguli, L.~Lovitt, J.~Kernion, A.~Askell, Y.~Bai, S.~Kadavath, B.~Mann,
  E.~Perez, N.~Schiefer, K.~Ndousse, et~al.
\newblock Red teaming language models to reduce harms: Methods, scaling
  behaviors, and lessons learned.
\newblock \emph{arXiv preprint arXiv:2209.07858}, 2022.

\bibitem[Gittens et~al.(2017)Gittens, Achlioptas, and Mahoney]{gittens2017skip}
A.~Gittens, D.~Achlioptas, and M.~W. Mahoney.
\newblock Skip-gram- zipf+ uniform= vector additivity.
\newblock In \emph{Proceedings of the 55th Annual Meeting of the Association
  for Computational Linguistics (Volume 1: Long Papers)}, pages 69--76, 2017.

\bibitem[Gresele et~al.(2021)Gresele, Von~K{\"u}gelgen, Stimper, Sch{\"o}lkopf,
  and Besserve]{gresele2021independent}
L.~Gresele, J.~Von~K{\"u}gelgen, V.~Stimper, B.~Sch{\"o}lkopf, and M.~Besserve.
\newblock Independent mechanism analysis, a new concept?
\newblock \emph{Advances in Neural Information Processing Systems}, 34, 2021.

\bibitem[Gupta et~al.(2023)Gupta, Jegelka, Lopez-Paz, and
  Ahuja]{gupta2023context}
S.~Gupta, S.~Jegelka, D.~Lopez-Paz, and K.~Ahuja.
\newblock Context is environment.
\newblock \emph{arXiv e-prints}, pages arXiv--2309, 2023.

\bibitem[Gurnee et~al.(2023)Gurnee, Nanda, Pauly, Harvey, Troitskii, and
  Bertsimas]{gurnee2023finding}
W.~Gurnee, N.~Nanda, M.~Pauly, K.~Harvey, D.~Troitskii, and D.~Bertsimas.
\newblock Finding neurons in a haystack: Case studies with sparse probing.
\newblock \emph{arXiv preprint arXiv:2305.01610}, 2023.

\bibitem[Hernandez et~al.(2023)Hernandez, Li, and
  Andreas]{hernandez2023measuring}
E.~Hernandez, B.~Z. Li, and J.~Andreas.
\newblock Measuring and manipulating knowledge representations in language
  models.
\newblock \emph{arXiv preprint arXiv:2304.00740}, 2023.

\bibitem[Hyvarinen and Morioka(2016)]{hyvarinen2016unsupervised}
A.~Hyvarinen and H.~Morioka.
\newblock Unsupervised feature extraction by time-contrastive learning and
  nonlinear ica.
\newblock \emph{Advances in neural information processing systems}, 29, 2016.

\bibitem[Hyv{\"a}rinen and Oja(2000)]{hyvarinen2000independent}
A.~Hyv{\"a}rinen and E.~Oja.
\newblock Independent component analysis: algorithms and applications.
\newblock \emph{Neural networks}, 13\penalty0 (4-5):\penalty0 411--430, 2000.

\bibitem[Hyv{\"a}rinen and Pajunen(1999)]{hyvarinen1999nonlinear}
A.~Hyv{\"a}rinen and P.~Pajunen.
\newblock Nonlinear independent component analysis: Existence and uniqueness
  results.
\newblock \emph{Neural networks}, 12\penalty0 (3):\penalty0 429--439, 1999.

\bibitem[Hyvarinen et~al.(2002)Hyvarinen, Karhunen, and
  Oja]{hyvarinen2002independent}
A.~Hyvarinen, J.~Karhunen, and E.~Oja.
\newblock Independent component analysis.
\newblock \emph{Studies in informatics and control}, 11\penalty0 (2):\penalty0
  205--207, 2002.

\bibitem[Hyvarinen et~al.(2019)Hyvarinen, Sasaki, and
  Turner]{hyvarinen2019nonlinear}
A.~Hyvarinen, H.~Sasaki, and R.~Turner.
\newblock Nonlinear ica using auxiliary variables and generalized contrastive
  learning.
\newblock In \emph{The 22nd International Conference on Artificial Intelligence
  and Statistics}, pages 859--868. PMLR, 2019.

\bibitem[Hyv{\"a}rinen et~al.(2023)Hyv{\"a}rinen, Khemakhem, and
  Monti]{hyvarinen2023identifiability}
A.~Hyv{\"a}rinen, I.~Khemakhem, and R.~Monti.
\newblock Identifiability of latent-variable and structural-equation models:
  from linear to nonlinear.
\newblock \emph{arXiv preprint arXiv:2302.02672}, 2023.

\bibitem[Jiang and Aragam(2023)]{jiang2023learning}
Y.~Jiang and B.~Aragam.
\newblock Learning latent causal graphs with unknown interventions.
\newblock In \emph{Advances in Neural Information Processing Systems}, 2023.

\bibitem[Jiang et~al.(2023)Jiang, Aragam, and Veitch]{jiang2023uncovering}
Y.~Jiang, B.~Aragam, and V.~Veitch.
\newblock Uncovering meanings of embeddings via partial orthogonality.
\newblock \emph{Advances in Neural Information Processing Systems}, 2023.

\bibitem[Jiang et~al.(2024)Jiang, Rajendran, Ravikumar, Aragam, and
  Veitch]{jiang2024learning}
Y.~Jiang, G.~Rajendran, P.~Ravikumar, B.~Aragam, and V.~Veitch.
\newblock On the origins of linear representations in large language models.
\newblock \emph{arXiv preprint}, 2024.

\bibitem[Khemakhem et~al.(2020{\natexlab{a}})Khemakhem, Kingma, Monti, and
  Hyvarinen]{khemakhem2020variational}
I.~Khemakhem, D.~Kingma, R.~Monti, and A.~Hyvarinen.
\newblock Variational autoencoders and nonlinear ica: A unifying framework.
\newblock In \emph{International Conference on Artificial Intelligence and
  Statistics}, pages 2207--2217. PMLR, 2020{\natexlab{a}}.

\bibitem[Khemakhem et~al.(2020{\natexlab{b}})Khemakhem, Monti, Kingma, and
  Hyvarinen]{khemakhem2020ice}
I.~Khemakhem, R.~Monti, D.~Kingma, and A.~Hyvarinen.
\newblock Ice-beem: Identifiable conditional energy-based deep models based on
  nonlinear ica.
\newblock \emph{Advances in Neural Information Processing Systems},
  33:\penalty0 12768--12778, 2020{\natexlab{b}}.

\bibitem[Kim et~al.(2018)Kim, Wattenberg, Gilmer, Cai, Wexler, Viegas,
  et~al.]{kim2018interpretability}
B.~Kim, M.~Wattenberg, J.~Gilmer, C.~Cai, J.~Wexler, F.~Viegas, et~al.
\newblock Interpretability beyond feature attribution: Quantitative testing
  with concept activation vectors (tcav).
\newblock In \emph{International conference on machine learning}, pages
  2668--2677. PMLR, 2018.

\bibitem[Kivva et~al.(2021)Kivva, Rajendran, Ravikumar, and
  Aragam]{kivva2021learning}
B.~Kivva, G.~Rajendran, P.~Ravikumar, and B.~Aragam.
\newblock Learning latent causal graphs via mixture oracles.
\newblock \emph{Advances in Neural Information Processing Systems},
  34:\penalty0 18087--18101, 2021.

\bibitem[Kivva et~al.(2022)Kivva, Rajendran, Ravikumar, and
  Aragam]{kivva2022identifiability}
B.~Kivva, G.~Rajendran, P.~Ravikumar, and B.~Aragam.
\newblock Identifiability of deep generative models without auxiliary
  information.
\newblock \emph{Advances in Neural Information Processing Systems},
  35:\penalty0 15687--15701, 2022.

\bibitem[Kong et~al.(2023)Kong, Xie, Yao, Zheng, Chen, Stojanov, Akinwande, and
  Zhang]{kong2023partial}
L.~Kong, S.~Xie, W.~Yao, Y.~Zheng, G.~Chen, P.~Stojanov, V.~Akinwande, and
  K.~Zhang.
\newblock Partial identifiability for domain adaptation.
\newblock \emph{arXiv preprint arXiv:2306.06510}, 2023.

\bibitem[Lachapelle et~al.(2022)Lachapelle, Rodr{\'{\i}}guez, Sharma, Everett,
  Priol, Lacoste, and Lacoste{-}Julien]{lachapelle2022disentanglement}
S.~Lachapelle, P.~Rodr{\'{\i}}guez, Y.~Sharma, K.~Everett, R.~L. Priol,
  A.~Lacoste, and S.~Lacoste{-}Julien.
\newblock Disentanglement via mechanism sparsity regularization: {A} new
  principle for nonlinear {ICA}.
\newblock In B.~Sch{\"{o}}lkopf, C.~Uhler, and K.~Zhang, editors, \emph{1st
  Conference on Causal Learning and Reasoning, CLeaR 2022, Sequoia Conference
  Center, Eureka, CA, USA, 11-13 April, 2022}, volume 177 of \emph{Proceedings
  of Machine Learning Research}, pages 428--484. {PMLR}, 2022.
\newblock URL \url{https://proceedings.mlr.press/v177/lachapelle22a.html}.

\bibitem[LeCun et~al.(2015)LeCun, Bengio, and Hinton]{lecun2015deep}
Y.~LeCun, Y.~Bengio, and G.~Hinton.
\newblock Deep learning.
\newblock \emph{nature}, 521\penalty0 (7553):\penalty0 436--444, 2015.

\bibitem[Levy and Goldberg(2014)]{levy2014neural}
O.~Levy and Y.~Goldberg.
\newblock Neural word embedding as implicit matrix factorization.
\newblock \emph{Advances in neural information processing systems}, 27, 2014.

\bibitem[Li et~al.(2022)Li, Hopkins, Bau, Vi{\'e}gas, Pfister, and
  Wattenberg]{li2022emergent}
K.~Li, A.~K. Hopkins, D.~Bau, F.~Vi{\'e}gas, H.~Pfister, and M.~Wattenberg.
\newblock Emergent world representations: Exploring a sequence model trained on
  a synthetic task.
\newblock \emph{arXiv preprint arXiv:2210.13382}, 2022.

\bibitem[Li et~al.(2023)Li, Patel, Vi{\'e}gas, Pfister, and
  Wattenberg]{li2023inference}
K.~Li, O.~Patel, F.~Vi{\'e}gas, H.~Pfister, and M.~Wattenberg.
\newblock Inference-time intervention: Eliciting truthful answers from a
  language model.
\newblock \emph{arXiv preprint arXiv:2306.03341}, 2023.

\bibitem[Li et~al.(2020)Li, Hooi, and Lee]{li2019identifying}
S.~Li, B.~Hooi, and G.~H. Lee.
\newblock Identifying through flows for recovering latent representations.
\newblock In \emph{8th International Conference on Learning Representations,
  {ICLR} 2020, Addis Ababa, Ethiopia, April 26-30, 2020}. OpenReview.net, 2020.
\newblock URL \url{https://openreview.net/forum?id=SklOUpEYvB}.

\bibitem[Lin et~al.(2021)Lin, Hilton, and Evans]{lin2021truthfulqa}
S.~Lin, J.~Hilton, and O.~Evans.
\newblock Truthfulqa: Measuring how models mimic human falsehoods.
\newblock \emph{arXiv preprint arXiv:2109.07958}, 2021.

\bibitem[Lippe et~al.(2023)Lippe, Magliacane, L{\"o}we, Asano, Cohen, and
  Gavves]{lippe2023biscuit}
P.~Lippe, S.~Magliacane, S.~L{\"o}we, Y.~M. Asano, T.~Cohen, and E.~Gavves.
\newblock Biscuit: Causal representation learning from binary interactions.
\newblock \emph{arXiv preprint arXiv:2306.09643}, 2023.

\bibitem[Liu et~al.(2022)Liu, Zhang, Gong, Gong, Huang, Hengel, Zhang, and
  Shi]{liu2022identifying}
Y.~Liu, Z.~Zhang, D.~Gong, M.~Gong, B.~Huang, A.~v.~d. Hengel, K.~Zhang, and
  J.~Q. Shi.
\newblock Identifying weight-variant latent causal models.
\newblock \emph{arXiv preprint arXiv:2208.14153}, 2022.

\bibitem[Locatello et~al.(2019)Locatello, Bauer, Lucic, Raetsch, Gelly,
  Sch{\"o}lkopf, and Bachem]{locatello2019challenging}
F.~Locatello, S.~Bauer, M.~Lucic, G.~Raetsch, S.~Gelly, B.~Sch{\"o}lkopf, and
  O.~Bachem.
\newblock Challenging common assumptions in the unsupervised learning of
  disentangled representations.
\newblock In \emph{international conference on machine learning}, pages
  4114--4124. PMLR, 2019.

\bibitem[Loshchilov and Hutter(2017)]{cosine}
I.~Loshchilov and F.~Hutter.
\newblock {SGDR:} stochastic gradient descent with warm restarts.
\newblock In \emph{5th International Conference on Learning Representations,
  {ICLR} 2017, Toulon, France, April 24-26, 2017, Conference Track
  Proceedings}. OpenReview.net, 2017.
\newblock URL \url{https://openreview.net/forum?id=Skq89Scxx}.

\bibitem[Lu et~al.(2021)Lu, Wu, Hern{\'a}ndez-Lobato, and
  Sch{\"o}lkopf]{lu2021invariant}
C.~Lu, Y.~Wu, J.~M. Hern{\'a}ndez-Lobato, and B.~Sch{\"o}lkopf.
\newblock Invariant causal representation learning for out-of-distribution
  generalization.
\newblock In \emph{International Conference on Learning Representations}, 2021.

\bibitem[McGrath et~al.(2022)McGrath, Kapishnikov, Toma{\v{s}}ev, Pearce,
  Wattenberg, Hassabis, Kim, Paquet, and Kramnik]{mcgrath2022acquisition}
T.~McGrath, A.~Kapishnikov, N.~Toma{\v{s}}ev, A.~Pearce, M.~Wattenberg,
  D.~Hassabis, B.~Kim, U.~Paquet, and V.~Kramnik.
\newblock Acquisition of chess knowledge in alphazero.
\newblock \emph{Proceedings of the National Academy of Sciences}, 119\penalty0
  (47):\penalty0 e2206625119, 2022.

\bibitem[Meng et~al.(2022)Meng, Bau, Andonian, and Belinkov]{meng2022locating}
K.~Meng, D.~Bau, A.~Andonian, and Y.~Belinkov.
\newblock Locating and editing factual associations in gpt.
\newblock \emph{Advances in Neural Information Processing Systems},
  35:\penalty0 17359--17372, 2022.

\bibitem[Mikolov et~al.(2013)Mikolov, Yih, and Zweig]{mikolov2013linguistic}
T.~Mikolov, W.-t. Yih, and G.~Zweig.
\newblock Linguistic regularities in continuous space word representations.
\newblock In \emph{Proceedings of the 2013 conference of the north american
  chapter of the association for computational linguistics: Human language
  technologies}, pages 746--751, 2013.

\bibitem[Moran et~al.(2022)Moran, Sridhar, Wang, and
  Blei]{moran2022identifiable}
G.~E. Moran, D.~Sridhar, Y.~Wang, and D.~Blei.
\newblock Identifiable deep generative models via sparse decoding.
\newblock \emph{Transactions on Machine Learning Research}, 2022.

\bibitem[Moschella et~al.(2022)Moschella, Maiorca, Fumero, Norelli, Locatello,
  and Rodola]{moschella2022relative}
L.~Moschella, V.~Maiorca, M.~Fumero, A.~Norelli, F.~Locatello, and E.~Rodola.
\newblock Relative representations enable zero-shot latent space communication.
\newblock \emph{arXiv preprint arXiv:2209.15430}, 2022.

\bibitem[Nakano et~al.(2021)Nakano, Hilton, Balaji, Wu, Ouyang, Kim, Hesse,
  Jain, Kosaraju, Saunders, et~al.]{nakano2021webgpt}
R.~Nakano, J.~Hilton, S.~Balaji, J.~Wu, L.~Ouyang, C.~Kim, C.~Hesse, S.~Jain,
  V.~Kosaraju, W.~Saunders, et~al.
\newblock Webgpt: Browser-assisted question-answering with human feedback.
\newblock \emph{arXiv preprint arXiv:2112.09332}, 2021.

\bibitem[Nanda et~al.(2023)Nanda, Lee, and Wattenberg]{nanda2023emergent}
N.~Nanda, A.~Lee, and M.~Wattenberg.
\newblock Emergent linear representations in world models of self-supervised
  sequence models.
\newblock \emph{arXiv preprint arXiv:2309.00941}, 2023.

\bibitem[Olah(2022)]{olah2022}
C.~Olah.
\newblock Mechanistic interpretability, variables, and the importance of
  interpretable bases.
\newblock
  \emph{\url{https://transformer-circuits.pub/2022/mech-interp-essay/index.html}},
  2022.

\bibitem[Ouyang et~al.(2022)Ouyang, Wu, Jiang, Almeida, Wainwright, Mishkin,
  Zhang, Agarwal, Slama, Ray, et~al.]{ouyang2022training}
L.~Ouyang, J.~Wu, X.~Jiang, D.~Almeida, C.~Wainwright, P.~Mishkin, C.~Zhang,
  S.~Agarwal, K.~Slama, A.~Ray, et~al.
\newblock Training language models to follow instructions with human feedback.
\newblock \emph{Advances in Neural Information Processing Systems},
  35:\penalty0 27730--27744, 2022.

\bibitem[Park et~al.(2023)Park, Choe, and Veitch]{park2023linear}
K.~Park, Y.~J. Choe, and V.~Veitch.
\newblock The linear representation hypothesis and the geometry of large
  language models.
\newblock \emph{arXiv preprint arXiv:2311.03658}, 2023.

\bibitem[Pearl(2009)]{pearl2009causality}
J.~Pearl.
\newblock \emph{Causality}.
\newblock Cambridge university press, 2009.

\bibitem[Pennington et~al.(2014)Pennington, Socher, and
  Manning]{pennington2014glove}
J.~Pennington, R.~Socher, and C.~D. Manning.
\newblock Glove: Global vectors for word representation.
\newblock In \emph{Proceedings of the 2014 conference on empirical methods in
  natural language processing (EMNLP)}, pages 1532--1543, 2014.

\bibitem[Peters et~al.(2017)Peters, Janzing, and
  Sch{\"o}lkopf]{peters2017elements}
J.~Peters, D.~Janzing, and B.~Sch{\"o}lkopf.
\newblock \emph{Elements of causal inference: foundations and learning
  algorithms}.
\newblock The MIT Press, 2017.

\bibitem[Radford et~al.(2015)Radford, Metz, and
  Chintala]{radford2015unsupervised}
A.~Radford, L.~Metz, and S.~Chintala.
\newblock Unsupervised representation learning with deep convolutional
  generative adversarial networks.
\newblock \emph{arXiv preprint arXiv:1511.06434}, 2015.

\bibitem[Radford et~al.(2017)Radford, Jozefowicz, and
  Sutskever]{radford2017learning}
A.~Radford, R.~Jozefowicz, and I.~Sutskever.
\newblock Learning to generate reviews and discovering sentiment.
\newblock \emph{arXiv preprint arXiv:1704.01444}, 2017.

\bibitem[Radford et~al.(2021)Radford, Kim, Hallacy, Ramesh, Goh, Agarwal,
  Sastry, Askell, Mishkin, Clark, et~al.]{radford2021learning}
A.~Radford, J.~W. Kim, C.~Hallacy, A.~Ramesh, G.~Goh, S.~Agarwal, G.~Sastry,
  A.~Askell, P.~Mishkin, J.~Clark, et~al.
\newblock Learning transferable visual models from natural language
  supervision.
\newblock In \emph{International conference on machine learning}, pages
  8748--8763. PMLR, 2021.

\bibitem[Rafailov et~al.(2023)Rafailov, Sharma, Mitchell, Ermon, Manning, and
  Finn]{rafailov2023direct}
R.~Rafailov, A.~Sharma, E.~Mitchell, S.~Ermon, C.~D. Manning, and C.~Finn.
\newblock Direct preference optimization: Your language model is secretly a
  reward model.
\newblock \emph{arXiv preprint arXiv:2305.18290}, 2023.

\bibitem[Raghu et~al.(2017)Raghu, Gilmer, Yosinski, and
  Sohl-Dickstein]{raghu2017svcca}
M.~Raghu, J.~Gilmer, J.~Yosinski, and J.~Sohl-Dickstein.
\newblock Svcca: Singular vector canonical correlation analysis for deep
  understanding and improvement.
\newblock \emph{stat}, 1050:\penalty0 19, 2017.

\bibitem[Rajendran et~al.(2021)Rajendran, Kivva, Gao, and
  Aragam]{rajendran2021structure}
G.~Rajendran, B.~Kivva, M.~Gao, and B.~Aragam.
\newblock Structure learning in polynomial time: Greedy algorithms, bregman
  information, and exponential families.
\newblock \emph{Advances in Neural Information Processing Systems},
  34:\penalty0 18660--18672, 2021.

\bibitem[Rajendran et~al.(2023)Rajendran, Reizinger, Brendel, and
  Ravikumar]{rajendran2023interventional}
G.~Rajendran, P.~Reizinger, W.~Brendel, and P.~Ravikumar.
\newblock An interventional perspective on identifiability in gaussian lti
  systems with independent component analysis.
\newblock \emph{arXiv preprint arXiv:2311.18048}, 2023.

\bibitem[Reimers and Gurevych(2019)]{reimers-2019-sentence-bert}
N.~Reimers and I.~Gurevych.
\newblock Sentence-bert: Sentence embeddings using siamese bert-networks.
\newblock In \emph{Proceedings of the 2019 Conference on Empirical Methods in
  Natural Language Processing}. Association for Computational Linguistics, 11
  2019.
\newblock URL \url{https://arxiv.org/abs/1908.10084}.

\bibitem[Rimsky et~al.(2023)Rimsky, Gabrieli, Schulz, Tong, Hubinger, and
  Turner]{rimsky2023steering}
N.~Rimsky, N.~Gabrieli, J.~Schulz, M.~Tong, E.~Hubinger, and A.~M. Turner.
\newblock Steering llama 2 via contrastive activation addition.
\newblock \emph{arXiv preprint arXiv:2312.06681}, 2023.

\bibitem[Sch{\"o}lkopf and von K{\"u}gelgen(2022)]{scholkopf2022statistical}
B.~Sch{\"o}lkopf and J.~von K{\"u}gelgen.
\newblock From statistical to causal learning.
\newblock In \emph{Proceedings of the International Congress of Mathematicians
  (ICM)}. EMS Press, July 2022.

\bibitem[Sch{\"o}lkopf et~al.(2021)Sch{\"o}lkopf, Locatello, Bauer, Ke,
  Kalchbrenner, Goyal, and Bengio]{scholkopf2021towards}
B.~Sch{\"o}lkopf, F.~Locatello, S.~Bauer, N.~R. Ke, N.~Kalchbrenner, A.~Goyal,
  and Y.~Bengio.
\newblock Toward causal representation learning.
\newblock \emph{Proceedings of the IEEE}, 109\penalty0 (5):\penalty0 612--634,
  2021.
\newblock arXiv:2102.11107.

\bibitem[Schut et~al.(2023)Schut, Tomasev, McGrath, Hassabis, Paquet, and
  Kim]{schut2023bridging}
L.~Schut, N.~Tomasev, T.~McGrath, D.~Hassabis, U.~Paquet, and B.~Kim.
\newblock Bridging the human-ai knowledge gap: Concept discovery and transfer
  in alphazero.
\newblock \emph{arXiv preprint arXiv:2310.16410}, 2023.

\bibitem[Seonwoo et~al.(2019)Seonwoo, Park, Kim, and Oh]{seonwoo2019additive}
Y.~Seonwoo, S.~Park, D.~Kim, and A.~Oh.
\newblock Additive compositionality of word vectors.
\newblock In \emph{Proceedings of the 5th Workshop on Noisy User-generated Text
  (W-NUT 2019)}, pages 387--396, 2019.

\bibitem[Shen et~al.(2022)Shen, Liu, Dong, Lian, Chen, and
  Zhang]{shen2022weakly}
X.~Shen, F.~Liu, H.~Dong, Q.~Lian, Z.~Chen, and T.~Zhang.
\newblock Weakly supervised disentangled generative causal representation
  learning.
\newblock \emph{Journal of Machine Learning Research}, 23:\penalty0 1--55,
  2022.

\bibitem[Shuster et~al.(2021)Shuster, Poff, Chen, Kiela, and
  Weston]{shuster2021retrieval}
K.~Shuster, S.~Poff, M.~Chen, D.~Kiela, and J.~Weston.
\newblock Retrieval augmentation reduces hallucination in conversation.
\newblock \emph{arXiv preprint arXiv:2104.07567}, 2021.

\bibitem[Sorrenson et~al.(2020)Sorrenson, Rother, and
  K{\"o}the]{sorrenson2020disentanglement}
P.~Sorrenson, C.~Rother, and U.~K{\"o}the.
\newblock Disentanglement by nonlinear ica with general incompressible-flow
  networks (gin).
\newblock \emph{arXiv preprint arXiv:2001.04872}, 2020.

\bibitem[Spirtes et~al.(2000)Spirtes, Glymour, and
  Scheines]{spirtes2000causation}
P.~Spirtes, C.~N. Glymour, and R.~Scheines.
\newblock \emph{Causation, prediction, and search}.
\newblock MIT press, 2000.

\bibitem[Squires and Uhler(2022)]{squires2022causal}
C.~Squires and C.~Uhler.
\newblock Causal structure learning: a combinatorial perspective.
\newblock \emph{Foundations of Computational Mathematics}, pages 1--35, 2022.

\bibitem[Squires et~al.(2023)Squires, Seigal, Bhate, and
  Uhler]{seigal2022linear}
C.~Squires, A.~Seigal, S.~S. Bhate, and C.~Uhler.
\newblock Linear causal disentanglement via interventions.
\newblock In A.~Krause, E.~Brunskill, K.~Cho, B.~Engelhardt, S.~Sabato, and
  J.~Scarlett, editors, \emph{International Conference on Machine Learning,
  {ICML} 2023, 23-29 July 2023, Honolulu, Hawaii, {USA}}, volume 202 of
  \emph{Proceedings of Machine Learning Research}, pages 32540--32560. {PMLR},
  2023.
\newblock URL \url{https://proceedings.mlr.press/v202/squires23a.html}.

\bibitem[Subramani et~al.(2022)Subramani, Suresh, and
  Peters]{subramani2022extracting}
N.~Subramani, N.~Suresh, and M.~E. Peters.
\newblock Extracting latent steering vectors from pretrained language models.
\newblock \emph{arXiv preprint arXiv:2205.05124}, 2022.

\bibitem[Szegedy et~al.(2013)Szegedy, Zaremba, Sutskever, Bruna, Erhan,
  Goodfellow, and Fergus]{szegedy2013intriguing}
C.~Szegedy, W.~Zaremba, I.~Sutskever, J.~Bruna, D.~Erhan, I.~Goodfellow, and
  R.~Fergus.
\newblock Intriguing properties of neural networks.
\newblock \emph{arXiv preprint arXiv:1312.6199}, 2013.

\bibitem[Talon et~al.(2023)Talon, Lippe, James, Del~Bue, and
  Magliacane]{talon2023towards}
D.~Talon, P.~Lippe, S.~James, A.~Del~Bue, and S.~Magliacane.
\newblock Towards the reusability and compositionality of causal
  representations.
\newblock In \emph{Causal Representation Learning Workshop at NeurIPS 2023},
  2023.

\bibitem[Taori et~al.(2023)Taori, Gulrajani, Zhang, Dubois, Li, Guestrin,
  Liang, and Hashimoto]{taori2023alpaca}
R.~Taori, I.~Gulrajani, T.~Zhang, Y.~Dubois, X.~Li, C.~Guestrin, P.~Liang, and
  T.~B. Hashimoto.
\newblock Alpaca: A strong, replicable instruction-following model.
\newblock \emph{Stanford Center for Research on Foundation Models.
  https://crfm. stanford. edu/2023/03/13/alpaca. html}, 3\penalty0
  (6):\penalty0 7, 2023.

\bibitem[Tenney et~al.(2019)Tenney, Das, and Pavlick]{tenney2019bert}
I.~Tenney, D.~Das, and E.~Pavlick.
\newblock Bert rediscovers the classical nlp pipeline.
\newblock \emph{arXiv preprint arXiv:1905.05950}, 2019.

\bibitem[Tigges et~al.(2023)Tigges, Hollinsworth, Geiger, and
  Nanda]{tigges2023linear}
C.~Tigges, O.~J. Hollinsworth, A.~Geiger, and N.~Nanda.
\newblock Linear representations of sentiment in large language models.
\newblock \emph{arXiv preprint arXiv:2310.15154}, 2023.

\bibitem[Touvron et~al.(2023)Touvron, Lavril, Izacard, Martinet, Lachaux,
  Lacroix, Rozi{\`e}re, Goyal, Hambro, Azhar, et~al.]{touvron2023llama}
H.~Touvron, T.~Lavril, G.~Izacard, X.~Martinet, M.-A. Lachaux, T.~Lacroix,
  B.~Rozi{\`e}re, N.~Goyal, E.~Hambro, F.~Azhar, et~al.
\newblock Llama: Open and efficient foundation language models.
\newblock \emph{arXiv preprint arXiv:2302.13971}, 2023.

\bibitem[Trager et~al.(2023)Trager, Perera, Zancato, Achille, Bhatia, and
  Soatto]{trager2023linear}
M.~Trager, P.~Perera, L.~Zancato, A.~Achille, P.~Bhatia, and S.~Soatto.
\newblock Linear spaces of meanings: compositional structures in
  vision-language models.
\newblock In \emph{Proceedings of the IEEE/CVF International Conference on
  Computer Vision}, pages 15395--15404, 2023.

\bibitem[Turner et~al.(2023)Turner, Thiergart, Udell, Leech, Mini, and
  MacDiarmid]{turner2023activation}
A.~Turner, L.~Thiergart, D.~Udell, G.~Leech, U.~Mini, and M.~MacDiarmid.
\newblock Activation addition: Steering language models without optimization.
\newblock \emph{arXiv preprint arXiv:2308.10248}, 2023.

\bibitem[Varici et~al.(2022)Varici, Shanmugam, Sattigeri, and
  Tajer]{varici2022intervention}
B.~Varici, K.~Shanmugam, P.~Sattigeri, and A.~Tajer.
\newblock Intervention target estimation in the presence of latent variables.
\newblock In \emph{Uncertainty in Artificial Intelligence}, pages 2013--2023.
  PMLR, 2022.

\bibitem[Varici et~al.(2023)Varici, Acarturk, Shanmugam, Kumar, and
  Tajer]{varici2023score}
B.~Varici, E.~Acarturk, K.~Shanmugam, A.~Kumar, and A.~Tajer.
\newblock Score-based causal representation learning with interventions.
\newblock \emph{arXiv preprint arXiv:2301.08230}, 2023.

\bibitem[Vaswani et~al.(2017)Vaswani, Shazeer, Parmar, Uszkoreit, Jones, Gomez,
  Kaiser, and Polosukhin]{vaswani2017attention}
A.~Vaswani, N.~Shazeer, N.~Parmar, J.~Uszkoreit, L.~Jones, A.~N. Gomez,
  {\L}.~Kaiser, and I.~Polosukhin.
\newblock Attention is all you need.
\newblock \emph{Advances in neural information processing systems}, 30, 2017.

\bibitem[Von~K{\"u}gelgen et~al.(2021)Von~K{\"u}gelgen, Sharma, Gresele,
  Brendel, Sch{\"o}lkopf, Besserve, and Locatello]{von2021self}
J.~Von~K{\"u}gelgen, Y.~Sharma, L.~Gresele, W.~Brendel, B.~Sch{\"o}lkopf,
  M.~Besserve, and F.~Locatello.
\newblock Self-supervised learning with data augmentations provably isolates
  content from style.
\newblock \emph{Advances in Neural Information Processing Systems}, 34, 2021.

\bibitem[von K{\"u}gelgen et~al.(2023)von K{\"u}gelgen, Besserve, Liang,
  Gresele, Keki{\'c}, Bareinboim, Blei, and
  Sch{\"o}lkopf]{von2023nonparametric}
J.~von K{\"u}gelgen, M.~Besserve, W.~Liang, L.~Gresele, A.~Keki{\'c},
  E.~Bareinboim, D.~M. Blei, and B.~Sch{\"o}lkopf.
\newblock Nonparametric identifiability of causal representations from unknown
  interventions.
\newblock In \emph{Advances in Neural Information Processing Systems}, 2023.

\bibitem[Wang et~al.(2022)Wang, Variengien, Conmy, Shlegeris, and
  Steinhardt]{wang2022interpretability}
K.~Wang, A.~Variengien, A.~Conmy, B.~Shlegeris, and J.~Steinhardt.
\newblock Interpretability in the wild: a circuit for indirect object
  identification in gpt-2 small.
\newblock \emph{arXiv preprint arXiv:2211.00593}, 2022.

\bibitem[Wang et~al.(2021)Wang, Blei, and Cunningham]{wang2021posterior}
Y.~Wang, D.~Blei, and J.~P. Cunningham.
\newblock Posterior collapse and latent variable non-identifiability.
\newblock \emph{Advances in Neural Information Processing Systems},
  34:\penalty0 5443--5455, 2021.

\bibitem[Wang et~al.(2023)Wang, Gui, Negrea, and Veitch]{wang2023concept}
Z.~Wang, L.~Gui, J.~Negrea, and V.~Veitch.
\newblock Concept algebra for score-based conditional model.
\newblock In \emph{ICML 2023 Workshop on Structured Probabilistic Inference
  $\{$$\backslash$\&$\}$ Generative Modeling}, 2023.

\bibitem[Wei et~al.(2022)Wei, Wang, Schuurmans, Bosma, Xia, Chi, Le, Zhou,
  et~al.]{wei2022chain}
J.~Wei, X.~Wang, D.~Schuurmans, M.~Bosma, F.~Xia, E.~Chi, Q.~V. Le, D.~Zhou,
  et~al.
\newblock Chain-of-thought prompting elicits reasoning in large language
  models.
\newblock \emph{Advances in Neural Information Processing Systems},
  35:\penalty0 24824--24837, 2022.

\bibitem[Willetts and Paige(2021)]{willetts2021don}
M.~Willetts and B.~Paige.
\newblock I don't need $\mathbf{u}$: Identifiable non-linear ica without side
  information.
\newblock \emph{arXiv preprint arXiv:2106.05238}, 2021.

\bibitem[Yang et~al.(2021)Yang, Liu, Chen, Shen, Hao, and Wang]{Yang_2021_CVPR}
M.~Yang, F.~Liu, Z.~Chen, X.~Shen, J.~Hao, and J.~Wang.
\newblock Causalvae: Disentangled representation learning via neural structural
  causal models.
\newblock In \emph{Proceedings of the IEEE/CVF Conference on Computer Vision
  and Pattern Recognition (CVPR)}, pages 9593--9602, June 2021.

\bibitem[Zhang and Nanda(2023)]{zhang2023towards}
F.~Zhang and N.~Nanda.
\newblock Towards best practices of activation patching in language models:
  Metrics and methods.
\newblock \emph{arXiv preprint arXiv:2309.16042}, 2023.

\bibitem[Zhang et~al.(2023{\natexlab{a}})Zhang, Squires, Greenewald,
  Srivastava, Shanmugam, and Uhler]{zhang2023identifiability}
J.~Zhang, C.~Squires, K.~Greenewald, A.~Srivastava, K.~Shanmugam, and C.~Uhler.
\newblock Identifiability guarantees for causal disentanglement from soft
  interventions.
\newblock \emph{arXiv preprint arXiv:2307.06250}, 2023{\natexlab{a}}.

\bibitem[Zhang et~al.(2023{\natexlab{b}})Zhang, Du, Huang, Wang, Wang, Fang,
  and Pechenizkiy]{zhang2023interpretable}
Y.~Zhang, Y.~Du, B.~Huang, Z.~Wang, J.~Wang, M.~Fang, and M.~Pechenizkiy.
\newblock Interpretable reward redistribution in reinforcement learning: A
  causal approach.
\newblock In \emph{Thirty-seventh Conference on Neural Information Processing
  Systems}, 2023{\natexlab{b}}.

\bibitem[Zheng et~al.(2022)Zheng, Ng, and Zhang]{zheng2022identifiability}
Y.~Zheng, I.~Ng, and K.~Zhang.
\newblock On the identifiability of nonlinear ica: Sparsity and beyond.
\newblock \emph{Advances in Neural Information Processing Systems},
  35:\penalty0 16411--16422, 2022.

\bibitem[Zimmermann et~al.(2021)Zimmermann, Sharma, Schneider, Bethge, and
  Brendel]{zimmermann2021contrastive}
R.~S. Zimmermann, Y.~Sharma, S.~Schneider, M.~Bethge, and W.~Brendel.
\newblock Contrastive learning inverts the data generating process.
\newblock In \emph{International Conference on Machine Learning}, pages
  12979--12990. PMLR, 2021.

\bibitem[Zou et~al.(2023)Zou, Phan, Chen, Campbell, Guo, Ren, Pan, Yin,
  Mazeika, Dombrowski, et~al.]{zou2023representation}
A.~Zou, L.~Phan, S.~Chen, J.~Campbell, P.~Guo, R.~Ren, A.~Pan, X.~Yin,
  M.~Mazeika, A.-K. Dombrowski, et~al.
\newblock Representation engineering: A top-down approach to ai transparency.
\newblock \emph{arXiv preprint arXiv:2310.01405}, 2023.

\end{thebibliography}
\bibliographystyle{abbrvnat}

\newpage
\appendix
\onecolumn

\section{Proofs of the main results}
\label{sec: proofs}
In this appendix we provide the proofs of our results, in particular the proof of our main result, Theorem~\ref{thm: main}. However, as a warm-up we first start in Appendix~\ref{app:ivae_proof} with a proof of the simpler result that can be shown based on the iVAE approach. In Appendix~\ref{app:main_proof}
we prove Theorem~\ref{thm: main} and in Appendix~\ref{app:lemmas_proof} we prove the additional lemmas that appear in the paper.

\subsection{Proof of identifiability with $2n+1$ environments}\label{app:ivae_proof}
As a warm-up and to provide a connection to earlier results we show here how to obtain
identifiability by adapting the iVAE framework to our context.
Indeed, our mathematical setting is related to the setting used in \cite{khemakhem2020variational} in the sense that the  environments are generated by modulation with certain exponential families.
 Therefore, we can essentially apply their proof techniques to prove identifiability (with some modifications), albeit this requires the suboptimal number of $2m+1$ environments (there are two sufficient statistics
    for the Gaussian distribution).
\begin{theorem}\label{thm:ivae}
Suppose data satisfies Assumption \ref{as:mixing}, \ref{as: lin_ind_atomic_concepts}, and
\ref{as:gaussianity} and the environment statistics matrix  $\Lda$ has rank $2n$.
Assume we know the number of atoms $n$. Then identifiability in the sense of Definition~\ref{def:identifiability} holds.
\end{theorem}
We remark that the rank condition can only be satisfied for $2n+1$ environments (observational distribution and $2n$ concept conditional distributions.
For this theorem the assumption that the filtering distribution is always the same is not necessary. Instead we could consider variances $(\sigma_k^e)^2$ depending on environment $e$ and row $k$,
i.e., the filtering distribution $q_{(\sigma_k^e)^2}$ is Gaussian with varying variance.
The generalization of the   environment-concept matrix $M\in \RR^{m\times n}$ is given by
\begin{align}\label{eq:env_conc2}
    M_{ei} = \begin{cases}
        \frac{1}{(\sigma^e_k)^2} &\text{if $i\in S^e$ and row $k$ of $A^e$ is $a_i$}
        \\
        0&\text{otherwise.}
    \end{cases}
\end{align}
Similarly the generalization of the environment-valuation matrix $B\in \RR^{m\times n}$ is given by
\begin{align}\label{eq:env_valuation2}
    B_{ei} = \begin{cases}
        \frac{b^e_{k}}{(\sigma^e_k)^2} &\text{if $i\in S^e$ and row $k$ of $A^e$ is $a_i$,}
        \\
        0&\text{otherwise.}
    \end{cases}
\end{align}

We now prove Theorem~\ref{thm:ivae}. We use essentially the same ideas as in the proof of Theorem 1 in
\citet{khemakhem2020variational} (followed by the same reasoning as in \citet{sorrenson2020disentanglement,kivva2022identifiability} but since our concepts are not axis aligned and we only extract some information about the mixing we give a complete proof.
\begin{proof}[Proof of Theorem~\ref{thm:ivae}]

Suppose there are $2$ sets of parameters that generate the same data $X^0, X^1, \ldots, X^m$. Denote by $\alt{.}$ the latter set of parameters, e.g., $X^e$ is distributed as $\alt{f}(\alt{Z}^e)$ where $\alt{Z}^e \in \RR^{\alt{d_z}}$ corresponds to the concept class $\alt{C^e}$ with distribution $\alt{Z}^e \sim \alt{p}^e$ and the same distribution is generated by $f(Z^e)$
where $f$ and $\alt{f}$ are injective and differentiable.
Let $\calC = \{a_1, \ldots, a_n\}$ be the set of atomic concepts in the first setting and let $\alt{\calC} = \{\alt{a}_1, \ldots, \alt{a}_n\}$ be the set of atomic concepts in the second setting (here we use that
$n$ is assumed to be known).
We also consider the transition function $\p=\alt{f}^{-1}f$ and in the following we always
write $\alt{Z}=\p(Z)$. The equality
$f(Z^e)\indistribution X^e\indistribution \alt{f}(\alt{Z}^e)$ implies $\p(Z^e)\indistribution \alt{Z}^e$.
This implies that for all environments $e$
\begin{align}
     p^e(Z) = |\det J_{\p^{-1}}| \cdot \alt{p}^e(\alt{Z})
\end{align}
Taking the logarithm and subtracting this for some $e=1,\ldots, m$ from the base distribution we obtain
\begin{align}\label{eq:eqality_log_dens_diff}
   \ln(p(Z))- \ln(p^e(Z)) = \ln(\alt{p}(\alt{Z}))-\ln(\alt{p}^e(\alt{Z})).
\end{align}
Using the definition \eqref{eq:concept_measure} we can rewrite for some constants $c_e$ and $c'_e$
\begin{align}
\begin{split}
  \ln(p(Z))-   \ln(p^e(Z)) &=  \sum_{k=1}^{\dim(C_e)} \frac{(A^e Z^e-b^e)_k^2}{2(\sigma^e_k)^2} - c'_e
  \\
  &= \sum_{i=1}^n \left(\frac{1}{2} M_{ei} \langle a_i, Z^e\rangle^2 - B_{ei} \langle a_i,Z^e\rangle\right)-c_e.
\end{split}
\label{eqn: logodds}
\end{align}
Here we used the environment-concept matrix and the environment-valuation matrix in the second step which were
defined in \eqref{eq:env_conc} and \eqref{eq:env_valuation} (in \eqref{eq:env_conc2} and
\eqref{eq:env_valuation2} for varying variance).
We define the vector $\bs{p}(Z)$ with components $\bs{p}_e(Z)= \ln(p(Z))-   \ln(p^e(Z))$.
Then we find the relation
\begin{align}
    \bs{p}(Z)=
    \frac12 M
    \begin{pmatrix}
       \langle a_1, Z\rangle^2
       \\
       \vdots
       \\
       \langle a_n, Z\rangle^2
    \end{pmatrix}
    -B
    \begin{pmatrix}
       \langle a_1, Z\rangle
       \\
       \vdots
       \\
       \langle a_n, Z\rangle
    \end{pmatrix}.
\end{align}
Together with \eqref{eq:eqality_log_dens_diff} we conclude that
\begin{align}\label{eq:eq_log_dense}
      \frac12 M
    \begin{pmatrix}
       \langle a_1, Z\rangle^2
       \\
       \vdots
       \\
       \langle a_n, Z\rangle^2
    \end{pmatrix}
    -B
    \begin{pmatrix}
       \langle a_1, Z\rangle
       \\
       \vdots
       \\
       \langle a_n, Z\rangle
    \end{pmatrix}
    =
       \frac12 \alt{M}
    \begin{pmatrix}
       \langle \alt{a}_1, \alt{Z}\rangle^2
       \\
       \vdots
       \\
       \langle \alt{a}_n, \alt{Z}\rangle^2
    \end{pmatrix}
    -\alt{B}
    \begin{pmatrix}
       \langle \alt{a}_1, \alt{Z}\rangle
       \\
       \vdots
       \\
       \langle \alt{a}_n, \alt{Z}\rangle
    \end{pmatrix}
\end{align}
Since by assumption $\alt{\Lambda}=(\alt{M},\alt{B})\in \RR^{m\times 2n}$ has rank $2n$ there is a vector $v$
such that $v^\top \alt{M} = 0$ and $v^\top \alt{B} = -\bs{e}_i$ ($\bs{e}_i\in\RR^{d_z}$ denotes the $i$-th standard basis
vector).
Thus we find that
\begin{align}
    \langle \alt{a}_i,\alt{Z}\rangle =
     \frac12 v^\top {M}
    \begin{pmatrix}
       \langle {a}_1, {Z}\rangle^2
       \\
       \vdots
       \\
       \langle {a}_n, {Z}\rangle^2
    \end{pmatrix}
    -v^\top {B}
    \begin{pmatrix}
       \langle {a}_1, {Z}\rangle
       \\
       \vdots
       \\
       \langle {a}_n, {Z}\rangle
    \end{pmatrix}.
\end{align}
In other words $\langle \alt{a}_i,\alt{Z}\rangle$ can be expressed as a quadratic polynomial in ${Z}$.
We apply the same reasoning for $\langle \alt{a}_i,\alt{Z}\rangle^2$, i.e., pick a vector $v'$ such that $\frac12 v'^\top \alt{M}=\bs{e}_i$ and
$ v'^\top \alt{B} = 0$ to obtain a relation
\begin{align}\label{eq:square_exression}
     \langle \alt{a}_i,\alt{Z}\rangle^2 =
     \sum_j \eta_j \langle {a}_j,{Z}\rangle^2 + \ell({Z})
\end{align}
for some coefficients $\eta_j$ and some affine function $\ell$  of ${Z}$.
The following reasoning is now the same as in \citet{kivva2022identifiability, sorrenson2020disentanglement}.
 We thus find that $\langle \alt{a}_i,\alt{Z}\rangle$ and its square can be written as polynimials
 of degree at most $2$ in $Z$.
This implies that in fact $\langle \alt{a}_i,\alt{Z}\rangle$ is an affine function of ${Z}$ (otherwise its
square would be a quartic polynomial), i.e.,
we can write
\begin{align}\label{eq:linear_relation}
    \langle \alt{a}_i,\alt{Z}\rangle = \sum_j \lambda_j \langle {a}_j,{Z}\rangle + C_i
    = \langle \sum_j \lambda_j {a}_j,{Z}\rangle + C_i.
\end{align}
 Equating the square of this relation with \eqref{eq:square_exression}
and taking the gradient with respect to ${Z}$ (as a polynomial the function is differentiable)
we find
\begin{align}
   2 \sum_j \eta_j {a}_j \langle {a}_j, {Z}\rangle + w
   =  2\sum_j \lambda_j {a}_j \langle  \sum_j \lambda_j {a}_j, {Z}\rangle+w'
\end{align}
for two vectors $w$ and $w'$. The equality (for ${Z}=0$) implies $w=w'$.
Now linear independence of ${a}_j$ implies that for each $r$
\begin{align}
    \eta_{r} {a}_{r} = \lambda_{r}\sum_j \lambda_j {a}_j.
\end{align}
Applying linear independence again we conclude that either $\lambda_{r}=0$
or $\lambda_j=0$ for all $j\neq r$. This implies that there is at most one $r$ such that
$\lambda_{r}\neq 0$. The relation \eqref{eq:linear_relation}
and the bijectivity of $\p$ implies that there is exactly on $r(i)$ such that
$\lambda_{r(i)}\neq 0$ and therefore
\begin{align}
    \langle \alt{a}_i,\alt{Z}\rangle = \lambda_{r(i)} \langle {a}_{r(i)},{Z}\rangle + C_i.
\end{align}
Applying the same argument in the reverse direction we conclude that there is a permutation $\pi\in S_n$
such that
\begin{align}
    \langle \alt{a}_{\pi(i)},\alt{Z}\rangle = \lambda_{i} \langle {a}_{i},{Z}\rangle + C_i.
\end{align}
By linear independence we can find an invertible linear map $T$ such that
\begin{align}\label{eq:linear_a_a_tilde}
    \alt{a}_{\pi(i)}^\top = a_i^\top T^{-1}
\end{align}
(i.e, $T^\top\alt{a}_{\pi(i)}=a_i$) and a
vector $w\in \RR^{d_z}$ (the $a_i$ are linearly independent) such that
\begin{align}\label{eq:idenfity_atoms}
      \langle \alt{a}_{\pi(i)},\alt{Z}\rangle = \lambda_{i} (\langle {a}_{i},{Z}\rangle + \langle a_i,w\rangle).
\end{align}
In particular the relations \eqref{eq:ident_atoms1} and \eqref{eq:ident_atoms2} hold.
Now it is straightforward to see that if $i\in S^e$, i.e., $a_i$ is a row
of $A^e$ then $\alt{a}_{\pi(i)}$ is a row of $\alt{A}^e$ and vice versa. Indeed, this follows from \eqref{eq:eq_log_dense}
for environment $e$ together with \eqref{eq:idenfity_atoms} and linear independence of the atoms.
Therefore we conclude from \eqref{eq:linear_a_a_tilde} that there is a permutation $P^e$ such that
\begin{align}
    \alt{A}^e=P^e A^eT^{-1}.
\end{align}
Moreover, \eqref{eq:idenfity_atoms} then implies setting $Z=f^{-1}(x)$, $\alt{Z}=\alt{f}^{-1}(x)$
\begin{align}
    \alt{A}^e \alt{f}^{-1}(x) = \Lambda^eP^e  A^e (f^{-1}(x)+w)
\end{align}
holds for the same permutation matrix $P^e$ and a diagonal matrix $\Lambda^e$ whose diagonal entries can be related to \eqref{eq:idenfity_atoms}.
Let us  assume now that row $k$ of $A^e$ is $a_i$ and row $k'$ of $\alt{A}^e$
is $\alt{a}_{\pi(i)}$.
Now we consider the subspace $H\subset \RR^{d_z}$ containing all $Z$ such that $\langle Z,a_j\rangle=0$
for $j\neq i$. Via \eqref{eq:idenfity_atoms} this implies that $\langle \alt{a}_j, \alt{Z}\rangle$ is constant for $j\neq \pi(i)$.
Then we conclude from \eqref{eq:eq_log_dense}  that for $Z\in H$
\begin{align}
    \frac{(\langle a_i, Z\rangle - b^e_k)^2 }{2(\sigma^e_k)^2}
    =
    \frac{(\langle \alt{a}_{\pi(i)}, \alt{Z}\rangle - \alt{b}^e_{k'})^2 }{2(\alt{\sigma}^e_{k'})^2} +c^e_k
\end{align}
for some constant $c^e_k$.
Using \eqref{eq:idenfity_atoms} this implies that
\begin{align}
     \frac{(\langle a_i, Z\rangle - b^e_k )^2}{2(\sigma^e_k)^2}
    =
    \frac{(\lambda_{i} (\langle {a}_{i},{Z}\rangle + \langle a_i ,w\rangle) - \alt{b}^e_{k'})^2 }{2(\alt{\sigma}^e_{k'})^2} +c^e_k.
\end{align}
Comparing the quadratic term and the linear term (note that $\langle a_i, Z\rangle$ can take any value on $H$) we find
\begin{align}
\frac{1}{2(\sigma^e_k)^2}& = \frac{\lambda_i^2}{2(\alt{\sigma}^e_{k'})^2}\\
    -\frac{b^e_k}{2(\sigma^e_k)^2}
    &= -\frac{\lambda_i\alt{b}^e_{k'} - \lambda_i^2\langle a_i ,w\rangle }{2(\alt{\sigma}^e_{k'})^2}
\end{align}
Combining the equation we obtain
\begin{align}
    \alt{b}^e_{k'} = \lambda_i(b^e_k - \langle a_i ,w\rangle )
\end{align}
This implies then the relation
\begin{align}
    \alt{b}=\Lambda^e P^e(b+A^e w).
\end{align}

\end{proof}

\subsection{Proof of Theorem~\ref{thm: main}}\label{app:main_proof}
In this section we prove our main Theorem~\ref{thm: main}. The proof is structured in several steps:
First we remove the symmetries of the representation and derive the key relations underlying the proof. Then we show that we can identify the environment-concept matrix $M$ and then also  the valuations collected in $B$. Once this is done we can complete the proof.
We will need the following lemma to conclude the proof.
\begin{lemma}\label{le:equivalence}
    The relations \eqref{eq:ident1} and \eqref{eq:ident_atoms2} in Definition~\ref{def:identifiability} define an equivalence relation of representations if we assume that the underlying atoms form a linearly independent set.
\end{lemma}
The proof of this lemma can be found in Appendix~\ref{app:lemmas_proof}.
\begin{remk}
    Without the assumption on the underlying atoms the lemma is not true. In this case a slightly different scaling must be chosen (e.g., $(\Lambda^e)^{-1}\alt{b}^e=\Lambda^e P^eb^e-P^eA^ew$
    instead of $\alt{b}^e=\Lambda^e P^e(b^e-A^ew)$). Since our
    results address the case of atoms we used the simpler definition in the main paper.
\end{remk}
We can allow slightly more general filtering distributions where $q$ is Gaussian with
variance $\sigma_i^2$ if we filter on concept $i$, i.e., the variance needs to be constant
for different environments and the same atom but might depend on the atom.
The proof will cover this case, the simple case stated in the main paper is obtained by setting $\sigma_i^2=\sigma^2$. Some steps of the proof (e.g., the expressions for the difference of the log-densities) agree with the proof of Theorem~\ref{thm:ivae}. To keep the proof self contained we
repeat a few equations.

\begin{proof}[Proof of Theorem~\ref{thm: main}]

We proceed in several steps.

\paragraph{Step 1: Reduction to standard form.}
Let us first transform every possible data representation into a standard form.
Recall that we have the set of atomic concepts $\calC=\{a_1,\ldots, a_n\}$.
Recall that we defined the environment-concept matrix $M\in \RR^{m\times n}$
in \eqref{eq:env_conc} and note that the natural generalisation reads
\begin{align}
    M_{ei}=\begin{cases}
        \frac{1}{\sigma_i^2}\quad \text{if $a_i$ is a row of $A^e$,}
        \\
        0\quad \text{otherwise.}
    \end{cases}
\end{align}
We say that concept $a_n$ is conditioned on the environment $e$.
Note that the nonzero entries of row $e$ of $M$ encode the set $S^e$.
To pass from $A^e$ to its rows $a_i$ we assume that the $e$-th row of $A^e$ is $a_{i^e_j}$,
i.e., $a_{i^e_j} =(A^e)^\top e_j$.
Recall also consider the environment-valuation matrix $B$ which is given by
\begin{align}
    B_{ei}=\begin{cases}
        \frac{b^e_k}{\sigma_i^2}\quad \text{if $a_i$ is row $k$ of $A^e$,}
        \\
        0\quad \text{otherwise.}
    \end{cases}
\end{align}
Denoting by $q_{\sigma^2}$ the centered Gaussian distribution with variance $\sigma^2$ we find
 in environment $e$
\begin{align}
\begin{split}
  \ln(p(Z))-   \ln(p^e(Z)) &=  -\sum_{k=1}^{\dim(C_e)} \ln q_{(\sigma_k^e)^2}((A^e Z^e-b^e)_k)
  =\sum_{k=1}^{\dim(C_e)} \frac{(A^e Z^e-b^e)_k^2}{2(\sigma^e_k)^2} - c'_e
  \\
  &= \sum_{i=1}^n \frac{1}{2} M_{ei} \langle a_i, Z^e\rangle^2 - B_{ei} \langle a_i,Z^e\rangle-c_e.
\end{split}
\label{eq:log_density1}
\end{align}
Now we consider an invertible linear map $T:\RR^{d_z}\to \RR^{d_z}$ such that $T^{-\top} a_i= e_i$
for all $1\leq i\leq n$.
Such a map exists because we assume that the $a_i$ are linearly independent.
Moreover, we consider a shift vector $\lambda
\in \RR^{d_z}$ with $\lambda_i=0$ for $i>n$ which we fix later.
We define $\Sigma\in \RR^{d_z\times d_z}$ to be the diagonal matrix with entries $\Sigma_{ii}=\sigma_i$
for $1\leq i\leq n$ and $\Sigma_{ii}=1$ for $i>n$.
Now we consider the linear map $L(z)=\Sigma^{-1} Tz-\lambda$ and a  new representation given by
\begin{align}
\begin{split}
    \alts{z}=L(z), \quad \alts{f}=f\circ L^{-1}, \quad \alts{\calC}=\{\bs{e}_1,\ldots, \bs{e}_n\},
   \quad \alts{\sigma}_i = 1,\quad \alts{A}^{e}=A^eT^{-1}, \quad \alts{p}(\alt{z}) = p(L^{-1}\alt{z})|\det T^{-1}|.
   \end{split}
\end{align}
We also define
\begin{align}
    \alts{b}^e_k = \frac{b^e_k}{\sigma_i} - \lambda_i \quad \text{if row $k$ of $A^e$ is $a_i$.}
\end{align}
Define $\alts{M}$ and $\alts{B}$ in terms of $\alts{A}^e$, $\alts{b}^e$ and $\alts{\sigma}_i^2$ as before.
We remark that all entries of $\alts{M}$ are either 0 or 1
and note that
\begin{align}\label{eq:altsM}
    \alts{M}= M\diag(\sigma_1^2, \ldots, \sigma_n^2)\\
    \label{eq:altsB}
    \alts{B} = B\diag(\sigma_1^{-1},\ldots, \sigma_n^{-1}) - M\diag(\lambda_1,\ldots, \lambda_n).
\end{align}
We claim that this model generates the same observations as the original model.
By definition $L_\ast p = \alts{p}$ (as mentioned before, we slightly abuse notation and here refer to the distributions). Next, we calculate for any $\delta$
\begin{align}
\begin{split}
    -2\ln q_1(\langle \bs{e}_i, L(z)\rangle - \delta)
    &= (\langle \bs{e}_i, L(z)\rangle -\delta)^2
    \\
    &= (\langle \bs{e}_i, \Sigma Tz-\lambda\rangle - \delta)^2
    \\
    &= (\sigma_i^{-1}\langle T^\top \bs{e}_i, z\rangle-\lambda_i - \delta)^2
    \\
    &=\frac{(\langle a_i, z\rangle-\sigma_i\lambda_i - \sigma_i\delta)^2}{\sigma_i^2}
    \\
    &=-2 \ln q_{\sigma_i^2} (\langle a_i, z\rangle-\sigma_i\lambda_i - \sigma_i\delta).
\end{split}
\end{align}
Using this for $\delta = \alts{b}^e_k$ and some $k$ such that row $k$ of $A^e $ is $a_i$ we find
\begin{align}
      -2\ln q_1(\langle \bs{e}_i, L(z)\rangle -\alts{b}^e_k)
    &=-2 \ln q_{\sigma_i^2} (\langle a_i, z\rangle-\sigma_i\lambda_i - \sigma_i\alts{b}^e_k)
    =-2 \ln q_{\sigma_i^2} (\langle a_i, z\rangle-{b}^e_k).
\end{align}

\noindent
This then implies that for $\alt{z}=L(z)$
\begin{align}
    \prod_k q_1 ( (\alt{A}^{e} \alt{z}-\alt{b}^{e})_k)
    \propto \prod_k q_{\sigma_{k}^e}\left(\left( A^ez-b^{e}\right)_k\right).
\end{align}

Combining this with the  definition \eqref{eq:concept_measure}
and the definition $\alts{p}(\alt{z}) = p(L^{-1}\alt{z})|\det T^{-1}|$
we find that for $\alts{z}=L(z)$
\begin{align}
    \alts{p}^{e}(\alt{z})\propto p^{e}(z)
\end{align}
and thus $\alts{f}(\alts{Z}^{e})\indistribution f(Z^{e})\indistribution X^{e}$.
Moreover, one directly sees that
the two representations are also equivalent in the sense of Definition~\ref{def:identifiability}.
We now fix the vector $\lambda$ such that each row of $\alts{B}$ has mean zero.
Finally, by changing the sign of $\alt{z}_i$ we can in addition assume that for every $i$ the first non-zero $\alts{B}_{ei}$ is positive.
Finally we remark that Assumption~\ref{ass:div} is still satisfied for $\alts{M}$ and $\alts{B}$.
Indeed, $w^\top M=0 $ implies $w^\top \alts{M}=0$ by \eqref{eq:altsM}.
But then $w^\top \alts{B}=w^\top B\diag(\sigma_1^{-1},\ldots, \sigma_n^{-1})$
by \eqref{eq:altsB}
which has all entries different from zero if this holds for $w^\top B$.
In the following we will therefore always assume that the representation satisfies the
properties of the $\alts{Z}$ variables and we remove the modifier in the following.
The plan is now to show that $M$ and $B$ can be identified up to permutations of the rows (under the fixed normalization we derived in this step) and then show that every two representations
with the same $M$ and $B$ can be identified.

\paragraph{Step 2: The key identity}
Let us here restate the key identity based on the difference of the log-densities.
As is common in identifiability results for multi-environment data with general mixing we consider the
difference in log densities. Consider
\begin{align}\label{eq:delta_log_probs}
\begin{split}
   \ln p^{0}(z) - \ln p^{e}(z)
  & =\sum_{i=1}^n \frac12{M}_{ei}\langle \bs{e}_i, z\rangle^2 -{B}_{ei}\langle \bs{e}_i, z\rangle - c'_e
   \\
   &=\sum_{i=1}^n \frac12{M}_{ei}z_i^2 -{B}_{ei}z_i- c'_e
  \end{split}
\end{align}
for some constant $c'_e$.
Those functions will play a crucial role in the following and we will denote
\begin{align}\label{eq:defOfg}
    g^e(z)=\ln p^{0}(z) - \ln p^{e}(z)
\end{align}
Note that since the log-density changes only by the Jacobian for pushforward measures we find that
\begin{align}
    g^e(z)=\ln p^{0}(z) - \ln p^{e}(z)
    =\ln p^{0}_X(f(z)) - \ln p^{e}_X(f(z))= G^e(f(z)) = G^e(x).
\end{align}
Note that the functions $G^e(x)$ can be estimated from the distributions of $X^e$.
We remark $X$ might be supported on a submanifold if $d_z$ and $d_x$ do not agree making the definition of the density subtle. But we can just consider any chart locally and consider the density of the
pushforward with respect to the Lebesgue measure. The resulting difference expressed in $G^e$ will be independent of the chart as the determinant cancels thus $G^e$ is a well defined function.
The relation
\begin{align}
    g^e(z)=G^e(f(z))=G^e(x)
\end{align}
will be crucial in the following because it shows that properties of $g^e$ are closely linked to the identifiable functions $G^e$.

\paragraph{Step 3: Identifiability of  environment-concept matrix}
Let us now show that we can identify which concepts are contained in which environment (up to relabeling of the concepts).
Recall that $S^e=\{i\in [n]: \text{$a_i$ is a row of $A^e$ } \} $
and we similarly define $S_T=\bigcup_{e\in T} S^e$ for all subsets $T\subset [m]$.
The main observation is that we can
identify $|S_T|=|\bigcup_{e\in T} S^e|$ for all subsets $T\subset [m]$.
To show this we consider the set
\begin{align}
    I_T = \argmin_{z} \sum_{e\in T} g^e(z).
\end{align}
Note that the function $g^e$ are convex functions, and they can be decomposed  as
sums of functions in $z_i$, i.e., for some functions $h^T_i$
\begin{align}
    \sum_{e\in T} g^e(z) = \sum_{i=1}^n h^T_i(z_i).
\end{align}
Now if $i\in S_T$ then $i\in S^e$ for some $e$ and thus $M_{ei}\neq 0$ for the $e$ and $h^T_i$
is the sum of quadratic function in $x_i$ which as a strictly convex function has a unique minimum
$z_i^T$.
On the other hand, if $i\notin S_T$ then $i\notin S^e$ for $e\in T$ and thus $M_{ei}=0$ for all
$e\in T$ and $h^T_i(z_i)=0$. Thus we conclude that
\begin{align}\label{eq:IT}
    I_T = \{ z\in \RR^{d_z}:\, z_i=z_i^T \text{ for $i\in S_T$}\}.
\end{align}
This is an affine subspace of dimension $d_z-|S_T|$.
The relations $G^e(f(z))=g^e(z)$ imply that
\begin{align}
    f(I_T)=\argmin_{x} \sum_{e\in T} G^e(x).
\end{align}
Note that $G^e(x)$ is identifiable from the datasets $X^e$ and thus the submanifold (by assumption on $f$)
$f(I_T)$ is identifiable
and by finding its dimension we obtain $d_z-|S_T|$. Since $d_z$ is the dimension of the data manifold $f(X)$
we can indeed identify $|S_T|$ for all $T\subset [m]$.
In particular, the total number of atomic concepts $n=|S_{[m]}|$ is identifiable (assuming that all atomic concepts are filtered upon at least once).
Now, it is a standard result that we can identify the matrix $M$ up to permutation of the
atomic concepts. %
Indeed, we can argue by induction in $m$ to show this. For $m=1$ we just have $|S^1|$ atomic concepts appearing
in environment $1$ and $n-|S^1|$ concepts not appearing.
For the induction step $m\to m+1$ we consider the sizes $|S_{T\cup \{m+1\}}|$ for $T\subset [m]$.
Applying the induction hypothesis we can complete $M_{ei}$ for all columns such that $M_{m+1,i}=1$.
Similarly, we can consider the sizes $|S_{T}|-|S_{T\cup \{m+1\}}|$ to identify
the matrix $M$ for concepts not used in environment $m+1$.

Thus, we can  and will assume after permuting the atomic concepts that $M$ is some fixed matrix.

\paragraph{Step 4: Identifiability of concept valuations}
Next, we show that we can also identify the matrix $B$.
We do this column by column, i.e., for one atomic concept after another.
Assume we consider atomic concept $i$.
Then we consider the
set $T_i=\{e: M_{ei}=0\}$ of concepts that not filter on atomic concept $i$.
By Assumption \ref{ass:div2} there is for every $i'\neq i$ an environment $e$ such that
$i'$ is filtered on, i.e.,
$M_{ei'}\neq 0$. This implies $S_{T_i}=[n]\setminus \{i\}$. Then we consider as in \eqref{eq:IT}
the set $I_{T_i}$ given by
\begin{align}
\label{eq:ITi}
     I_{T_i} = \{ z\in \RR^{d_z}:\, z_{i'}=z_{i'}^{T_i} \text{ for $i'\in [n]\setminus \{i\}$}\}.
\end{align}
Note that all $z_{i'}$ for $i\neq i'$ are constant on $I_{T_i}$.
Thus we find for any environment $e$  such that $i\in S^e$.
\begin{align}
\begin{split}\label{eq:ge_onIT}
    g^{e}(z)&=
    \sum_{j=1}^n \frac12 {M}_{ej}z_j^2 -{B}_{ej}z_j - c'_e
   \\
   &=
    \sum_{j\neq i}^n\frac12  {M}_{ej}z_j^2 -{B}_{ej}z_j - c'_e
    + \frac12 z_i^2 - B_{ei} z_i\\
    &=c_{T_i,e}+ \frac{1}{2}z_i^2-B_{ei}z_i
\end{split}
\end{align}
on $I_{T_i}$
for some constant $c_{T_i}$.

Now we consider two concepts $e_1 \neq e_2$ such that atomic concept $i$ is contained in these two environments. Then we consider
the set
\begin{align}
    I_{T_i}^{e_1} = \argmin_{z\in I_{T_i}}g^{e_1}(z)
    =  \{ z\in \RR^{d_z}:\, z_{i'}=z_{i'}^{T_i} \text{ for $i'\in [n]\setminus \{i\}$,
    $z_i=B_{e_1 i}$}\}
    .
\end{align}
Note that in the second equality we used that $g^{e_1}(z)$ depends on $z_i$ through
$z_i^2/2-B{e_1i}z_i$ so it is minimized at $B_{e_1i}$.
Now we find using \eqref{eq:ge_onIT}
\begin{align}
    \begin{split}
        \min_{z\in I_{T_i}^{e_1}} g^{e_2}(z) - \min_{I_{T_i}} g^{e_2}(z)
        &= \min_{z\in I_{T_i}^{e_1}} c_{T_i,e_2}+ \frac12 z_i^2 - B_{e_2i}z_i - \min_{I_{T_i}} \left(c_{T_i,e_2}+ \frac12 z_i^2 - B_{e_2i}z_i\right)
       \\
       &=  c_{T_i,e_2}+ \frac{1}{2} B_{e_1i}^2 - B_{e_1i}B_{e_2i} - \left(c_{T_i,e_2}+ \frac{1}{2}B_{e_2i}^2-B_{e_2i}^2\right)
      \\
      &=
      \frac{(B_{e_1i}-B_{e_2i})^2}{2}.
    \end{split}
\end{align}
As before, this quantity is identifiable from observations because $f(T_i)$ can be identified
and we can minimize $G^{e_2}(x)$ over $f(T_i)$.

This allows us to identify $B_{e_1i}-B_{e_2i}$ up to a sign. However, we can evaluate this expression over all pairs $e_1$ and $e_2$ and pick the one with the maximal difference. Then  all remaining
values $B_{ei}$ for $e$ such that $i$ is filtered on in $e$ must satisfy $B_{ei}\in [B_{e_1i},B_{e_2i}]$.
Together with identifiability of $|B_{ei}-B_{e_1i}|$ this allows us to identify all $B_{ei}$ up to
one sign indeterminacy and a constant shift. However, in the first step we ensured that
$\sum_{e} B_{ei}=0$ for all $i$ which determines the shift and the sign is fixed by our choice of making the first non-zero entry positive.
Thus, we can assume that our two representations have the same $M$ and $B$.

\paragraph{Step 5: Identifiability of concepts}
We are now ready to prove our identifiability result.

Assume we have two representations $Z^e$, $f$, $p$ and $\alt{Z}^e$, $\alt{f}$, and $\alt{p}$
such that the corresponding environment-concept and environment-valuation matrices agree, i.e.,
$M=\alt{M}$ and $B=\alt{B}$.
We consider the transition function $\p=\alt{f}^{-1}\circ f$ which is by assumption differentiable.
What we want to show is that $\p(z)_i = z_i$ for all $z\in \RR^{d_z}$ and $1\leq i\leq n$.
We now decompose $z=(z^c,z^o)$ into the concept part and the orthogonal part. We fix  $z^o\in \RR^{d_z-n}$ and define the function $\iota^o(z^c)=(z^c,z^o)$, the projection $\pi^c((z^c,z^o))=z^c$, and  $\p^o:\RR^n\to \RR^n$ given by
$\p^o(z^c)_i=\p(\iota^o(z^c)=\p((z^c,z^o))_i$.
Note that $\p^o$ is differentiable but not necessarily injective.
Let us denote by $\boldsymbol{g}: \RR^{d_z}\to \RR^m$ the function with coordinates $\boldsymbol{g}_e=g^e$
and similarly we define $\boldsymbol{G}:M\to \RR^d$.
Identifiability will be based on the crucial relation
\begin{align}\label{eq:bold_g}
    \boldsymbol{g}(\iota^o(z^c))=\boldsymbol{G}(f(\iota^o(z^c)))=
    \boldsymbol{G}(\alt{f}(\p^o(z^c)))=\boldsymbol{g}(\p^o(z^c)).
\end{align}
Here we used in the last step that $g^e$ is defined in terms of $M$ and $B$ and thus agrees for both
representations.
Note that $\boldsymbol{g}$ is just a quadratic function.
Differentiating we obtain
\begin{align}
    D_i  g^e(z)= M_{ei}z_i-B_{ei}.
\end{align}
Concisely this can be written as
\begin{align}
    D\boldsymbol{g}= M \diag(z_1,\ldots,z_n) - B.
\end{align}
Differentiating \eqref{eq:bold_g} we find
\begin{align}\label{eq:deriv_relation}
    M \diag(z_1,\ldots,z_n) - B = (M \diag(\alt{z}_1,\ldots,\alt{z}_n) - B) D\p^o(z^c).
\end{align}
Let $v$ be a vector as in Assumption~\ref{ass:div}. Denote by $M^+\in \RR^{n\times m}$ the pseudoinverse of $M$ which has rank $n$
because $M$ has. We consider the matrix $\alt{M^+}\in \RR^{n+1\times m}$ given by
\begin{align}
    \alt{M^+} = \begin{pmatrix}
        M^+\\
        v^\top
    \end{pmatrix}
\end{align}
Let us multiply the relation \eqref{eq:deriv_relation} by $\alt{M^+}$ and find that
\begin{align}\label{eq:relation_multiplied}
    \begin{pmatrix}
        z_1 & & 0
        \\
        &\ddots &
        \\
        0&& z_n
        \\
        0&\ldots & 0
    \end{pmatrix}
    -\alt{M^+} B =\left(
     \begin{pmatrix}
        \alt{z}_1 & & 0
        \\
        &\ddots &
        \\
        0&& \alt{z}_n
        \\
        0&\ldots & 0
    \end{pmatrix}
    -\alt{M^+} B \right)D\p^o(z^c)
\end{align}
Note that the first $n$ rows of the left hand side are $\diag(z_1,\ldots,z_n)-M^+B$.
This matrix is invertible for almost all values of $z^c=(z_1,\ldots,z_n)^\top$ because its determinant is a
non-zero  polynomial (the coefficient of the term $z_1\cdot\ldots z_n$ is 1) which vanishes only on a set of measure zero.
Outside of this set the left hand side of has rank $n$. Then the equality \eqref{eq:relation_multiplied}
implies that also the right hand side has rank $n$ and thus $D\p^o(z^c)$ has rank $n$ and thus is invertible.
For $z^c$ outside of this set there is up to scaling a unique vector $w\neq 0$ (depending on $z_1,\ldots, z_n$ such that
\begin{align}\label{eq:relation_z}
    w^\top\left(  \begin{pmatrix}
        z_1 & & 0
        \\
        &\ddots &
        \\
        0&& z_n
        \\
        0&\ldots & 0
    \end{pmatrix}
    -\alt{M^+} B \right)=0
\end{align}
From \eqref{eq:relation_multiplied} we conclude using the invertibility of $D\p^o(z^c)$
that
\begin{align}\label{eq:relation_ztilde}
     w^\top\left(  \begin{pmatrix}
        \alt{z}_1 & & 0
        \\
        &\ddots &
        \\
        0&& \alt{z}_n
        \\
        0&\ldots & 0
    \end{pmatrix}
    -\alt{M^+} B \right)=0.
\end{align}
Next, we claim that for almost all values of $z^c$ the vector $w$ has all entries different from 0 (this property is invariant under rescaling). Actually we need this only for entries 1 to $n$ but the case $n+1$
is a bit simpler so we show it first.
We show this by proving that for each entry $w_i$ there is only a null set of $z^c$ such that $w_i=0$.
Let $w=(w',0)$ for some $w'\in \RR^n$ and $w'\neq 0$, i.e., $w_{n+1}=0$. Then
\begin{align}
   0= w^\top\left(  \begin{pmatrix}
        z_1 & & 0
        \\
        &\ddots &
        \\
        0&& z_n
        \\
        0&\ldots & 0
    \end{pmatrix}
    -\alt{M^+} B \right)=
    w'^\top (\diag(z_1,\ldots,z_n) -M^+B)
\end{align}
But this implies that $\diag(z_1,\ldots,z_n) -M^+B$ has non-trivial kernel, i.e., does not have full rank
and we have seen above that this happens only for a subset of measure 0 of all $z^c$.
Next we show that the same is true if $w_1=0$. Decompose $0\neq w=(0,w')$.
Then we find
\begin{align}
       0= w^\top\left(  \begin{pmatrix}
        z_1 & & 0
        \\
        &\ddots &
        \\
        0&& z_n
        \\
        0&\ldots & 0
    \end{pmatrix}
    -\alt{M^+} B \right)=
    w'^\top \left(  \begin{pmatrix}
        0 &z_2 & 0 & 0
        \\
       \ldots & &\ddots &
        \\
        0&&  &z_n
        \\
        0&\ldots & \ldots& 0
    \end{pmatrix}
    -(\alt{M^+} B)_{2:(n+1)} \right)
\end{align}
Thus we conclude that the matrix on the right hand side is not invertible.
Its determinant is a polynomial in $z_2,\ldots, z_n$ and its highest degree term is
$\pm z_2\cdot \ldots\cdot  z_n \cdot (\alt{M^+} B)_{(n+1),1}$.
By definition of $\alt{M^+} B$ we find $(\alt{M^+} B)_{(n+1),1}=(v^\top B)_1\neq 0$ by Assumption~\ref{ass:div} (recall that we showed invariance of the assumption under the transformation of $M$ and $B$).
We find that the determinant is a non-zero polynomial and the set of its zeros is a set of measure 0
of all $z_2, \ldots, z_n$ but since it does not depend on $z_1$ this holds true for almost all $z^c$.
The same reasoning for $i=2,\ldots, n$ implies that
for every $i$ the set of $z^c$ such that $w_i=0$ is a set of measure zero.
We have therefore shown that for almost all $z^c$ the rank of the left hand side
of \eqref{eq:relation_multiplied} is $n$ and the corresponding vector $w\neq 0$ has all entries different from zero.
Subtracting \eqref{eq:relation_z} and \eqref{eq:relation_ztilde} we obtain
\begin{align}
    0=w^\top \begin{pmatrix}
        z_1 & & 0
        \\
        &\ddots &
        \\
        0&& z_n
        \\
        0&\ldots & 0
    \end{pmatrix}
    -
    w^\top  \begin{pmatrix}
        \alt{z}_1 & & 0
        \\
        &\ddots &
        \\
        0&& \alt{z}_n
        \\
        0&\ldots & 0
    \end{pmatrix}
    = \begin{pmatrix}
        w_1(z_1-\alt{z}_1),
        &
        \ldots
        &
        w_n (z_n-\alt{z}_n), 0
    \end{pmatrix}.
\end{align}
Now $w_i\neq 0$ implies $z_i=\alt{z}_i$. We conclude that for almost all $z^c$ the relation $\p^o(z^c)=z^c$
holds. By continuity this implies that the relation actually holds everywhere.
We conclude that $\pi^c\alt{f}^{-1}f((z^c,z^o))=z^c$ for a fixed $z^o$ but since $z^o$ was arbitrary the relation holds for all $z^o$ and all $z^c$.
Thus we conclude that for $1\leq i\leq n$
\begin{align}
    \langle \bs{e}_i, \alt{f}^{-1}(x)\rangle = \langle \bs{e}_i, \p(f^{-1}(x))\rangle
    = \langle \bs{e}_i, f^{-1}(x)\rangle
\end{align}
holds. This implies that those two representations satisfy \eqref{eq:ident1} and \eqref{eq:ident2}
(with $P^e=\Lambda^e=\mathrm{Id}$ and $T=\mathrm{Id}$).
But since this relation is an equivalence relation in our setting by Lemma~\ref{le:equivalence}
and since we showed equivalence to a representation in standard form in the first step
we conclude that also any two representations are related through \eqref{eq:ident1} and \eqref{eq:ident2}
thus finishing the proof.

\end{proof}

\subsection{Remaining proofs}
\label{app:lemmas_proof}
Here we prove the remaining auxiliary results.
\begin{proof}[Proof of Lemma~\ref{le:generic_assumption}]
Since $M\in \RR^{m\times n}$ has rank $n$ and $m=n+1$ there is exactly one vector $v\in \RR^m$ such
that $v^\top M = 0$ and $v\neq 0$. We claim that this vector has all entries different from zero. Indeed suppose $v_{m}=0$ which
then implies $v_{1:(m-1)}^\top M_{1:(m-1)}=0$.  But by assumption every $n\times n $ submatrix of $M$ is invertible
(this is equivalent to the rows being linearly independent) so we conclude that $v_{1:(m-1)}=0$ which is a contradiction to $v\neq 0$. The same reasoning applies to every entry.
Note that the assumption on $M$ implies that every column has at least one non-zero entry, i.e., every column of $B$ has one entry sampled from a continuous distribution. But then the probability that $v$ is orthogonal to a column is zero because this is a codimension 1 hyperplane of all valuations of this row (since  all entries of $v$ are non-zero).
\end{proof}
\begin{proof}[Proof of Lemma~\ref{le:equivalence}]
Reflexivity is obvious, just pick $T=\mathrm{Id}$, $w=0$, $\Lambda^e=P^e=\mathrm{Id}_{\dim(C^e)}$.
To show symmetry we first consider the atoms. Let $\tilde{T}=T^{-1}$ and $\tilde{\pi}=\pi^{-1}$.
Then
\begin{align}
    a_{\alt{\pi}(i)}^\top=a_{\pi^{-1}(i)}^\top T^{-1} T
    = \alt{a}_{\pi\circ\pi^{-1}(i)}\alt{T}^{-1}=\alt{a}_i\alt{T}^{-1}.
\end{align}
Let $\alt{w}$ be a vector such that for all $1\leq i\leq n$
\begin{align}\label{eq:prop_tilde_w}
    \langle a_i, w\rangle =- \frac{1}{\lambda_{i}}\langle \alt{a}_{\pi(i)}, \alt{w}\rangle.
\end{align}
Such a vector exists by linear independence of $\alt{a}_i$.
Let $\alt{\lambda}_i=\lambda_{\alt{\pi}(i)}^{-1}$. Then we find that the relation
\eqref{eq:ident_atoms2}, namely
\begin{align}
   \langle \alt{a}_{\pi(i)},\alt{f}^{-1}(x)\rangle &=  \lambda_i\left(  \langle a_i, f^{-1}(x)\rangle + \langle a_i,w\rangle\right)
\end{align}
implies
\begin{align}
\begin{split}
    \langle a_{\alt{\pi}(i)}, f^{-1}(x)\rangle &=
    \frac{1}{\lambda_{\alt{\pi}(i)}}  \langle \alt{a}_{\pi\circ \alt{\pi}(i)},\alt{f}^{-1}(x)\rangle
    -\langle a_{\alt{\pi}(i)},w\rangle
    =  \frac{1}{\lambda_{\alt{\pi}(i)}}  \langle \alt{a}_{i},\alt{f}^{-1}(x)\rangle+ \frac{1}{\lambda_{\alt{\pi}(i)}}\langle \alt{a}_{\pi\circ \alt{\pi}(i)}, \alt{w}\rangle
   \\
   &=\alt{\lambda}_i ( \langle \alt{a}_{i},\alt{f}^{-1}(x)\rangle+\langle \alt{a}_{i}, \alt{w}\rangle).
    \end{split}
\end{align}
It remains to be shown that this lifts to the concepts $C^e$.
We first note that the relation \eqref{eq:ident_atoms2} together with \eqref{eq:prop_tilde_w}
and \eqref{eq:ident1} implies that
\begin{align}
    \Lambda^e P^e A^e w = -\alt{A}^e\alt{w}.
\end{align}
Let $\alt{P}^{e}=(P^e)^{-1}$
and $\alt{\Lambda}^e = (P^e)^{-1} (\Lambda^e)^{-1}P^e$.
Then \eqref{eq:ident1} combined with the previous disply implies
\begin{align}\label{eq:ident1_proof_sym}
    A^e f^{-1}(x)=(P^e)^{-1}(\Lambda^e)^{-1}\alt{A}^e\alt{f}^{-1}(x)-A^e w
    = \alt{\Lambda}^e \alt{P}^e\alt{A}^e\alt{f}^{-1}(x)+(P^e)^{-1}(\Lambda^e)^{-1}\alt{A}\alt{w}
    =\alt{\Lambda}^e \alt{P}^e\alt{A}^e(\alt{f}^{-1}(x)+\alt{w}).
\end{align}
The relation
\begin{align}
    A^e = \alt{P}^e \alt{A}^e \alt{T}^{-1}
\end{align}
is a direct consequence of the definitions of $\alt{P}^e$ and $\alt{T}$ and \eqref{eq:ident2}
and the relation
\begin{align}
    b^e = \alt{\Lambda}^e \alt{P}^e(\alt{b}^e-\alt{A}^e w)
\end{align}
follows exactly as in \eqref{eq:ident1_proof_sym}.
The proof of transitivity is similar (first establish the relations on the atomic concepts then lift
it to $C^e$).
\end{proof}
\section{Comparison to Causal Representation Learning}\label{app:crl}
In this appendix we describe causal representation learning and discuss the similarities and differences between the viewpoint taken in this paper and the standard setting in causal representation learning.

Causal Representation Learning (CRL) \cite{scholkopf2021towards, scholkopf2022statistical} is a modern machine learning field that aims to learn representations of data that correspond to the true generative process.
More precisely, if we assume that data $X$ is generated as $X = f(Z)$ where $Z$ are latent causal factors and $f$ is some arbitrary nonlinearity, the goal is to learn $f$ as well as the distribution of $Z$.
To make this more aligned with reality, the latent variables $Z$ are assumed to have causal relationships among them. For instance, $Z_1, Z_2$ could correspond to shape and size of an object respectively and $f$ could correspond to the rendering of an image of the object. Then, using datasets of images, we wish to learn latent variables that correspond to shape and size respectively, along with the image rendering map.
CRL incorporates ideas from the field of causality \cite{spirtes2000causation, pearl2009causality, peters2017elements, rajendran2021structure, squires2022causal} into the field of latent variable models and is a generalization of nonlinear independent component analysis \cite{comon1994independent, hyvarinen2000independent, hyvarinen2002independent} and disentangled representation learning \cite{bengio2013representation,peters2017elements,lecun2015deep}. The field has seen a surge of advances in the last few years, e.g., \cite{
khemakhem2020variational,
kivva2021learning, falck2021multi, liu2022identifying, lachapelle2022disentanglement, buchholz2022function, moran2022identifiable, zimmermann2021contrastive, gresele2021independent, rajendran2023interventional, varici2023score, jiang2023learning, hyvarinen2023identifiability, talon2023towards}
and featured several workshops on Causal Representation Learning, among others at UAI 2022, CLeaR 2023 and NeurIPS 2023.
As motivated in \citet{scholkopf2021towards}, CRL enables many desiderata such as robustness, out of distribution generalization, and in addition enables planning and alignment.
CRL has been successfuly applied to many domains, already showing fascinating progress in genomics \cite{zhang2023identifiability, seigal2022linear} and holds great promise for other domains such as vision or text (see references above).

At the same time, research into foundation models for text, images and other fields has burgeoned in recent years. Significant progress has been made to enable robustness, prediction and generalization among other desiderata. Much of this research has been largely experimental and has not strictly followed the paradigms set forth in causal representation learning. We therefore endeavor to answer this discrepancy in this work with our unified framework.

To describe the intuition, we will re-emphasize the goal of CRL below, because our work critically departs from this goal.

\begin{quote}
    \em{Causal representation learning aims to find the true underlying generative factors for a data distribution.}
\end{quote}
\noindent
The issue here is that attempts towards CRL often conflate the notion of \textit{causal generative factors} with \textit{human interpretable factors}. That is, it's tempting to imagine that generative factors for a dataset are natural causes that are interpretable to human beings. However, there is no reason that this should indeed be the case.

In our work, we take significant inspiration from the framework of causal representation learning and present a slightly relaxed paradigm that is weaker, but more general and also importantly, aligns better with many high-performance foundation models in the literature.
We now describe the setup of CRL more formally in \cref{sec: crl_setup}.
Then, in \cref{sec: crl_diff}, we discuss conceptual differences between causal representation learning and our framework.

\subsection{Formal setup}\label{sec: crl_setup}

We assume that we observe data $X \in \RR^{d_x}$ with the generative model $X = f(Z)$ where $Z \in \RR^{d_z}$ are the latent variables and $f$ is a deterministic mixing function.
The dataset $X$ is sampled from a distribution $p$ and the goal is to recover the mixing function $f$ as well as the distributions of the underlying latent variables $Z_1, \ldots, Z_{d_z}$. To this end, this problem is over-parameterized since multiple pairs of $Z$ and $f$ could fit the dataset apriori, so the field of representation learning makes various assumptions to learn this  model \textit{identifiably}. Here, identifiability is the notion that a unique set of parameters fit the model (up to trivial transformations). This makes the problem well-defined and feasible (however it could still be a hard problem to solve in practice).
Below, we informally summarize two classes of prior works that enable such identifiability guarantees.
\begin{enumerate}
    \item Disentangled representation learning: In this setting, we assume that the distributions of $Z_1, \ldots, Z_{d_z}$ are jointly independent.
    Different studies constrain the distribution of the variables $Z_1, \ldots, Z_{d_z}$, e.g., each $Z_i$ is independently sampled from $N(0, 1)$. This is also the setting studied in nonlinear independent component analysis \cite{comon1994independent, hyvarinen2000independent}.
    \item Causal Representation Learning: This setting is more general than the one above where we relax the independence assumption on the $Z_i$, instead we assume that they have causal relationships among them.
    For instance, they could satisfy a linear structural causal model with Gaussian noise, i.e., $Z = AZ +\eps, \eps \sim N(0, I)$ where $A$ encodes a weighted directed acyclic graph.
    As stated, this setting is more general then the above, since having no causal relationships (i.e., $A = 0$) essentially indicates joint independence.
\end{enumerate}

As explained earlier, in both these domains, a critical notion is that of identifiability \cite{khemakhem2020variational, d2022underspecification, wang2021posterior}, which posits that the given dataset(s) are diverse enough for the modeling assumptions, in order to ensure that a unique set of parameters fit the data.

It's folklore that the disentangled representation learning model is not identifiable if all $Z_i$ are Gaussian \cite{hyvarinen1999nonlinear, locatello2019challenging}.
However, under appropriate assumptions, e.g., distributional, sparsity or observed side-information, the model becomes identifiable, see e.g., \citet{khemakhem2020variational, hyvarinen2016unsupervised, brehmer2022weakly, shen2022weakly, lachapelle2022disentanglement, moran2022identifiable, zheng2022identifiability, kivva2022identifiability, buchholz2022function, zimmermann2021contrastive, gresele2021independent, rajendran2023interventional}.
In addition, various works have proposed methods to learn them (with modest success) \cite{falck2021multi, willetts2021don, dilokthanakul2016deep, Yang_2021_CVPR, li2019identifying, cuiaggnce, buchholz2022function, buchholz2023learning}.

\subsection{Conceptual differences}\label{sec: crl_diff}

In this section, we highlight the conceptual differences between causal representation learning and our framework.

\paragraph{Are causal generative concepts necessarily interpretable?}

While it's tempting to imagine the true causal generative factors as being naturally  human interpretable concepts (e.g., intuitive abstractions such as shape, size or color of an object), there's no obvious reason why this should be the case.
The perspective that we take in this work is that the number of true generative factors could be prohibitively large so that attempting to learn them is infeasible, whereas the number of desired human-interpretable concepts is much smaller.
Moreover, we constantly come up with new concepts of interest since human-interpretable concepts are constantly evolving, e.g., the concept of mobile phones was not existent 100 years ago, but is a valid concept to learn now. Therefore, as opposed to working with a rigid model as in causal representation learning, we take the approach of working with a dynamic representation learning model.

\paragraph{Number of environments needed}
When the ground truth generative process has ambient latent dimention $d_z$, for causal representation learning to be feasible, we usually require $d_z$ environments or datasets. For instance, in the iVAE setting above \cite{khemakhem2020variational}, we require $d_zk + 1 \ge d_z + 1$ environments.
This is indeed necessary, as counterexamples show.
However, it's not clear what the value of $d_z$ is for complex datasets, and it could potentially be prohibitively large.

But the question remains, do we need to learn the entire generative model for solving downstream tasks?
Along these lines, there is a tremendous research effort attempting to relax such requirements by imposing various inductive or domain biases and by building a theory of partial identifiability \cite{kivva2022identifiability, lippe2023biscuit, kong2023partial}.
This is for good reason, since even though it would be ideal to learn the full ground truth generative model, it may be too ambitious or prohibitively large and moreover it maybe not be necessary for most downstream tasks we care about, therefore it suffices to learn what's necessary.
On this note, the relevant task of learning only a subset of the generative latent variables is not easy as the latent variables interact in potentially complicated ways and research in this front is non-trivial and limited progress has been made.

In this work however, we show that if we only wish to learn $m \ll d_z$ concepts, it suffices to have $O(m)$ environments instead of $O(d_z)$ environments. Therefore, our results can be viewed as a result on partial identifiability with sublinear number of environments.

\paragraph{Multi-node interventions}

Multi-node interventions are an exciting area of study in CRL, since they are a natural extension of existing works and are more useful for modeling various real-life datasets where it can be hard to control precisely one factor of variation. This is easily incorporated in our setting by utilizing non-atomic concepts, since each non-atomic concept is a collection of vectors corresponding to atomic concepts and can be modified simultaneously by changing the valuation.

\paragraph{Conditional vs. interventional data}
In this work we focus on conditional data and identification of concept values, while most of the recent identifiability results focus on interventional data and identification of the causal structure \cite{squires2022causal,varici2022intervention,buchholz2023learning,jiang2023learning,von2021self}.
In general conditional data is more frequently available. Conditional data can be obtained by selection through filtering, e.g., patients that are admitted to different hospitals based on the severity of their condition or by the availability of label information as in the CLIP setting \cite{radford2021learning}.
Thus conditional data can be obtained by observing the system in different condtions. On the other hand
interventional data requires manipulation of the system which is more difficult to obtain in general. Moreover, many of the theoretical results require perfect interventions which can be achieved  by randomized control trials but are otherwise a very strong assumption. We note that in typical applications we do not think of concepts as being causal variables that are connected by a graph.
On the other hand, our framework can accomodate the case where concepts correspond to causal variables. Our identifiability result does not allow us to learn the causal structure. However, identifying the causal variables reduces the problem from causal representation learning to causal inference (with observational data) which is still a difficult problem but nevertheless substantially simpler than CRL.

\section{Alternate definitions of concept conditional measure}\label{app:alt_defns_concept_distributions}

In this section, we present alternate feasible definitions for data distributions than the one we introduced in \cref{sec: data_distributions}. While we went with the definition most suited for practice, these alternate definitions are also justifiable in different scenarios and are exciting avenues for further study.

We want to essentially define a concept $C$ via a conditional measure $p_C$ where the concept $C$
is identified with an affine subspace $C=\{Z\in \mathbb{R}^{d_z}: \, A^C Z = b^C\}$ for some $A^C\in \mathbb{R}^{k\times d_z}$, $b^C\in
\mathbb{R}^k$. We consider the shifted parallel linear subspace $C_0=\{Z: A^CZ=0\}$ and the orthogonal splitting
$\mathbb{R}^{d_z}=C_0\oplus V$. Suppose we have a distribution $q_V$ on the space $V$ which will typically be a Gaussian centred around
$v^C\in V$ which is the unique solution of $A^Cv^C=b^C$.
In addition we have a base distribution $p$ on $\RR^{d_z}$.
We will assume that all distributions have a smooth density so that conditional probabilities are pointwise well defined.
There are at least three ways to create the context conditional measure $p_C$.
\begin{enumerate}
    \item The first option is to enforce that the distribution of the $V$ marginal $p_C(v)=\int_{C_0}p_C(v,c)\,\mathrm{d}c$ exactly matches $q_V(v)$ while the in-plane distribution  $p_C(c|v=v_0)\propto p_C(c,v_0)$
    remains invariant, i.e., equals $p(c|v=v_0)$. Under this condition, there is a unique measure $p_C$ given by
    \begin{align*}
        p_C(c,v) \propto q_V(v) \frac{p(c,v)}{\int_{C_0} p(c',v)\,\mathrm{d}c'}.
    \end{align*}
    In other words, to get $(c,v)$ we sample $v\sim q_V$ and then $c\sim p(c|v)$ according to the conditional distribution.
    \item The second option is to again enforce the $V$ marginal but instead of keeping the in plane distribution we average over
    the $V$ space. Then we obtain
    \begin{align*}
        p_C(c,v) \propto q_V(v) \int_{V} p(c,v')\, \mathrm{d}v'.
    \end{align*}
    This corresponds (vaguely) to a $\mathrm{do}(v)$ operation from causal inference, i.e., we sample according to $p(v,c)$ and then do a random intervention on $v$ with target distribution $q_V$.
    \item The third option is to take a Bayesian standpoint. Then we view
    $p$ as a prior and $q_V$ as the context dependent acceptance probability, i.e., we sample by $p$ and then accept with probability $q_V$. Then we find
    \begin{align}
        p_C(c,v) = \frac{p(c,v)q_V(v)}{\int p(c,v)q_V(v)\,\mathrm{d}v\,\mathrm{d}c}\propto p(c,v)q_V(v).
    \end{align}
    This is probably the closest aligned to practice, so this is the one we study in this work.
    To justify this option, imagine the following scenario.
    If we wish to learn the concept of \textit{red color}, a first step would be to curate a dataset of red objects. To do this, we first consider a collection of photos of objects of varying color and then filter out the ones that look red. The concept conditional measure we define aligns with this process. To learn the actual red concept accurately, our theory predicts that it is sufficient to have additional datasets of objects that are not red, from which we can distinguish red objects, thereby learning the concept of red color.
\end{enumerate}

The next question is how to define the measure $q_V$.
When considering a single concept $A^CZ=b^C$ the most natural option to consider
$N(v^C, \sigma^2\mathrm{Id}_V)$ where $v^C \in V$ is the unique solution of $A^Cv^C=b^C$ and $\sigma>0$ is a positive constant. This is what we do in this work (note that $\sig^2$ can be set to $1$ by scaling the concept and valuation accordingly).

However, we can also use alternate definitions as suggested above. For instance, we can set $AZ\indistribution N(b^C,\mathrm{Id})$. Then
$Z\sim N(v^C, (A^\top A)^{-1})$. However, this runs into some technical issues we sketch (and leave to future work to handle this). Consider the intersection of multiple concepts $C^e$. In this case the concept space is given by the intersection
$C=\bigcap C^e$ and $C_0=\bigcap (C^e)_0$ and we have the orthogonal decomposition
$\RR^{d_z} = C_0\oplus \sum V^e$. In general the spaces $V^e$ are not necessarily orthogonal but it is reasonable to assume that the non-degeneracy condition
$\dim(\sum V^i)=\sum \dim(V^e)$ holds. Now set $V=\sum V^e$. If we choose just the standard normal distribution for $q_{V^e}$ we
can define just as in our approach
\begin{align}
    q_V\sim N(v^C, \sigma^2\mathrm{Id}_V).
\end{align}
The second option is to enforce that the marginals of $q_V$ agree with $q_{V^e}$, i.e., $q_V(\Pi_{V^e}(v)\in O)=q_{V^e}(O)$
for $O\subset V^e$.
This results in the set of equations for all $i$
\begin{align}
 A^e   \Sigma (A^e)^\top = \mathrm{Id}_{V^e}.
\end{align}
It is likely that this system has a unique solution when non-degeneracy holds for $V^e$ and this is clearly true for orthogonal spaces but it is not clear how to solve this in general.

\section{Inference-Time Intervention of Large Language Models}
\label{sec: iti}

In this section, we first briefly describe Large Language Models and the recent Inference-Time Intervention (ITI) technique proposed for LLM alignment, which we build on.
Then, we use our framework to provide better intuition on some intriguing observations about ITI, including why it works. And then we exploit our ideas to improve the performance of ITI by choosing the steering direction to be a matrix instead of a vector.

\subsection{Preliminaries}

\paragraph{Large Language Models (LLMs)}
LLMs are large models capable of generating meaningful text given a context sentence. Due to large-scale training, modern LLMs have shown remarkable capabilities and achieve expert-human-like performance in many benchmarks simultaneously. The architecture of many generative pre-trained transformers (GPT)-style LLMs consists of several transformer layers stacked on top of each other.
Since we'll be intervening on them during inference, we'll describe the transformer architecture \cite{vaswani2017attention,elhage2021mathematical} briefly here.
First, the sequence of input tokens (tokens are sub-word units) are encoded into a vector $x_0$ using a (learned) text embedding matrix and in many cases also a positional embedding matrix. Then, a series of transformer layers act on this vector which passes through a residual stream, to obtain vectors $x_0, x_1, \ldots, x_n$. The final vector $x_n$ is then decoded back into token probabilities with a (learned) unembedding matrix.
Each transformer layer consists of a multi-head attention mechanism and a standard multilayer perceptron, which captures the nonlinearity.

In the $l$th layer, each single multi-head attention mechanism can be described as
\begin{align*}
    x_{l + 1} = x_l + \sum_{h = 1}^H Q_l^hx_l^h, \qquad x_l^h = \textrm{Att}_l^h(P_l^hx_l)
\end{align*}

Here, $P_l^h$ and $Q_l^h$ are matrices that linearly map the vector to an activation space and back respectively, and $\textrm{Att}$ denotes the attention mechanism that allows communication across tokens.
Here, we have kept the notation consistent with \citet{li2023inference} for the sake of clarity.

In our setting, we consider the entire set of activations as the learnt latent vector $Z$. That is, the input is $x = x_0$ and the pre-trained model is essentially the function $f$ such that $f(x)$ consists of the concatenation of the vectors $\{x_l\}_{l \ge 1}$, the intermediate activations $\{x_l^h\}_{l \ge 0}$ and also the output of the linear transformations $\{P_l^hx_l\}_{l \ge 0}, \{Q_l^hx_l^h\}_{l \ge 0}$.
Our theory hinges on the assumption that pre-trained LLMs satisfy the linear representation hypothesis, that is, various relevant concepts can be realized via linear transformations of the latent transformation $f(x)$. Indeed, this has been empirically observed to hold in many prior works
\cite{burns2022discovering, tigges2023linear, nanda2023emergent, moschella2022relative, li2023inference, park2023linear, gurnee2023finding, jiang2024learning} (see also related works on geometry of representations \cite{jiang2023uncovering, jiang2024learning} and references therein). It's a fascinating question why such models trained with next token prediction loss also learn linear representations of various human-interpretable concepts such as sentiment, see \citet{jiang2024learning} for recent progress on this problem.

It's well-known that despite large-scale pretraining and subsequent improvement of pre-trained models via techniques like Reinforcement Learning with Human Feedback (RLHF) and Supervised Fine-Tuning (SFT) \cite{ouyang2022training, bai2022training, touvron2023llama}, significant issues still remain \cite{shuster2021retrieval}, e.g., the model can hallucinate or generate incorrect responses (even though the model \textit{knows} the correct response which can be extracted via other means, e.g., Chain-of-Thought prompting \cite{wei2022chain}).
Various methods have been proposed to fine-tune the models \cite{ouyang2022training, bai2022training, bai2022constitutional, touvron2023llama, rafailov2023direct} but many of them are expensive and time- and resource-intensive as they requires huge annotation and computation resources.
Therefore, more efficient techniques are highly desired, one of which is the category of methods known as activation patching.
activation patching (also called activation editing or activation engineering) \cite{hernandez2023measuring, wang2022interpretability, subramani2022extracting, turner2023activation, zou2023representation, zhang2023towards, li2022emergent, meng2022locating}.

\paragraph{Inference-Time Intervention, an activation patching method for truthfulness}

Activation patching is a simple minimally invasive technique to align LLMs to human-preferences. Specifically, given various concepts such as truthfulness, activation patching makes modifications to the model during inference time so that the desired concepts can be aligned.
This technique can be thought of as an application of the emerging field of mechanistic interpretability \cite{olah2022}, which aims to interpret the learnt latent vector in terms of human-interpretable concepts, thereby allowing us to reverse-engineer what large models learn.

Activation patching has many variants \cite{li2022emergent, hernandez2023measuring, meng2022locating}, but we'll focus on the simple technique of adding \textit{steering vectors} to various intermediate layers during intervention \cite{subramani2022extracting, turner2023activation, li2023inference, rimsky2023steering}. This means that during inference, the output activations are modified by adding a constant vector in order to promote alignment of some concept. The vector will be learnt independently based on separate training data.

In particular, a recent technique called Inference-Time Intervention (ITI) was proposed to do this for the specific concept of truthfulness.
ITI focuses on the activation heads $\{\textrm{Att}_l^h(P_l^hx_l)\}_{l \ge 0}$ and add to them steering vectors in order to promote truthfulness. To learn the steering vectors, a subset of the TruthfulQA dataset \cite{lin2021truthfulqa}, namely a dataset of questions $q_i$ with annotated true $(a_{i, j}, 0)$ and false answers $(a_{i, j}, 1)$, are prepared as $\{q_i, a_i, y_i\}_{i = 1, 2, \ldots}$. For each sample, the question and answer are concatenated as a pair and the corresponding activations of the heads $x_l^h$ (for the final token) are computed via forward passes.
Then, a linear probe $\textrm{sigmoid}(\ip{\theta}{x_l^h})$ is independently trained on each activation head to distinguish true from false answers. Finally, the top $K$ heads based on the accuracy of this classification task are chosen (for a tunable hyperparameter $K$) and the steering vector $\theta_l^h$ for the $h$-th head in layer $l$ is chosen to be
the mean difference of the activations between the true and false inputs. The intuition is that this direction roughly captures the direction towards truthfulness.

Formally, for the $h$th head of the $l$th layer, ITI adds the steering vector $\alpha \sig_l^h \theta_l^h$ so as to get
\begin{align*}
    x_{l + 1} = x_l + \sum_{h = 1}^H Q_l^h(x_l^h + \alpha \sig_l^h \theta_l^h), \qquad x_l^h = \textrm{Att}_l^h(P_l^hx_l)
\end{align*}
during inference.
Here, $\theta_l^h$ is the steering vector, $\sig_l^h$ is the standard deviation of the activations of this head along the chosen direction and $\alpha$ is a hyperparameter. That is, the activations are shifted along the truthful directions by a multiple of the standard deviation, and this is repeated autoregressively. Note that this does not depend on the specific GPT-like model being used. The intuition is that during inference, the activations are intervened upon to shift towards the truthful direction. The top $K$ heads are chosen to be minimally intrusive and also a design choice based on observations of the probing metrics.

\paragraph{Performance of ITI}

In \citet{li2023inference}, ITI was shown to significantly improve the truthfulness of various LLMs after having been trained on as few as a few dozen samples, compared to what's needed for Reinforcement Learning based techqniues \cite{ouyang2022training, ganguli2022red}.
ITI was evaluated on the TruthfulQA benchmark \cite{lin2021truthfulqa}, which is a hard adversarial benchmark to evaluate truthfulness of language models.
In particular, it contains 817 questions with a multiple-choice and generation tracks, spanning 38 categories such as logical falsehoods, conspiracies and common points of confusion. For the multiple-choice questions, the accuracy is determined by the conditional probabilities of candidate answers given the question. Evaluating the generation track questions is harder, and it is done by generating a model output and then evaluating it via a finetuned GPT-3-13B model \cite{lin2021truthfulqa, nakano2021webgpt}.
Moreover, the choice of the intervention strength $\alpha$ is calibrated so that it's neither too small (to promote truthfulness) nor
too large (to ensure the original capabilities of the LLM are not lost).
To check if the original capabilies are preserved, \cite{li2023inference} compute two additional quantities to measure how far the modified model deviates from the original model. These are the Cross-Entropy (CE) loss, which is standard in language modeling and the Kullback–Leibler divergence (KL div.) of the next token probabilities before and after intervention. To compute these quantities, a subset of Open Web Text is used \cite{radford2017learning}.
Finally, it was shown that ITI implemented on the LLaMA \cite{touvron2023llama}, Alpaca \cite{taori2023alpaca} and Vicuna \cite{vicuna2023} models significantly improved their performance on the TruthfulQA benchmark compared to the baseline models. Moreover, in many cases, it also beat other techniques such as few-shot prompting and supervised fine-tuning.
Please see \citet{li2023inference} for additional details.

\subsection{Interesting observations of ITI}
\label{sec: iti_observations}

While the elegant ITI technique was designed to align LLMs towards truthfulness in practice, it also raised fascinating and intriguing questions in mechanistic interpretability.
In addition to improving the technique of ITI itself, our work makes progress towards some of these questions via our framework.

\begin{enumerate}
    \item The authors of \citet{li2023inference} state in section $2$ that although the technique works well in practice, it's not clear what ITI does to the model's internal representations. In addition, prior works \cite{burns2022discovering, tigges2023linear, nanda2023emergent, moschella2022relative, park2023linear, jiang2024learning} have observed empirically that the latent representations learned by LLMs seem to have interpretable linear directions, which ITI exploits.
    We use our framework to illustrate in more detail one possible explanation of what ITI does to the model representations and why it works, in the next section.
    \item The authors visualize the geometry of ``truth'' representations in section 3.2 of their work via the following experiment: For the most significant head (layer 14, head 18), after finding the first truthful direction via the linear probing technique, they remove it and attempt to find a second probe orthogonal to the first. They find surprisingly that the second probe is also very informative, leading them to predict that the concept of ``truth'' lies in a subspace, not a single direction.
    Restated in our framework, the concept of truthfulness is a non-atomic concept (as per \cref{def: atomic}).
    This served as an inspiration for our proposed technique in the next section, where we propose to use steering matrices instead of steering vectors for LLM alignment.
    \item As $\al$ was increased, the authors observed that truthfulness of the model increased however helpfulness decreased.
    This suggests that the ``truthfulness'' and ``helpfulness'' concepts are not atomic (as per \cref{def: atomic}) however they share certain atomic concepts. We leave to future work the exciting question of mechanistically extracting such common atomic concepts.
\end{enumerate}

\subsection{The choice of the steering vector}
\label{sec: steering_vector}

We will now analyze the truthfulness concept via our framework and give more insight on why the mean of the differences is a reasonable choice of steering vector for ITI.
Based on our theory, we will then provide a modification to this choice that uses steering matrices instead of steering vectors.
Since this section is based on heuristics and informal assumptions, we will refrain from making any formal claims or analyses.
Indeed, a formal analysis of concepts in natural language is a hard problem in general and we do not attempt it here.
We conclude with ideas for potential extensions that're worth exploring in future work.

Denote the function $h$ to be the sequence of head activations $h(x) = (x_l^h)_{l, h} \in \RR^d$.
Note that while we can study general steering vectors for the entire latent space of representations $f(x)$ learned by LLMs as some works do, ITI focuses only on steering the head activations $h(x)$, so we will apply our framework to this subset representation space.
In addition, we will make the simplification that we neglect the effects of the steering vector from bottom layers towards the top layers, which we do because we are  dealing with sparse steering vectors and also, each single head shift is minor and does not in isolation change the behavior of the model as verified by experiments \cite{li2023inference}[Appendix B.1].

Applying our framework, we model the concept of truth via the concept matrix $A \in \RR^{d_C \times d}$ and two valuations $b_0, b_1 \in \RR^{d_C}$ corresponding to \textit{False} and \textit{True} respectively. In other words, the set of false sentences and true sentences lie respectively in
\begin{align*}
    \calS_{false} = \{x | Ah(x) = b_0\}, \qquad \calS_{true} = \{x | Ah(x) = b_1\}
\end{align*}

Note that they only approximately lie in these spaces because of our notion of concept conditional distribution. However, if we reasonably assume that the Gaussian concentration region is much smaller than the separation between these hyperplanes, then the rest of the arguments in this section should apply.

Now, a steering vector $\eta$ is a vector such that it moves the activations from the false space to the true space, while keeping other concepts unaffected. That is, if we pick a false sentence $x$, i.e., $Ah(x) = b_0$, then the steering vector $\eta \in \RR^d$ essentially steers the activations so that $A(h(x) + \eta) = b_1$. In other words, it moves the sentence from false to true. Indeed, many vectors $\eta$ do satisfy this equality, therefore the goal is to find an optimal $\eta$ that does not (significantly) affect other concepts of interest. Indeed, the ideal steering vector will be $A^+(b_1 - b_0)$ where $A^+$ is the pseudoinverse of $A$. This vector will precisely affect this concept space and will not affect the concept valuations for any concept orthogonal to $A$. However, the issue is that we do not know $A$ and therefore we will approximate this steering vector from training samples.

To this end, suppose we are given a collection of counterfactual sentence pairs $c_i^F, c_i^T$ which correspond to a false answer and a true answer for the same question $q_i$. Consider the $i$th counterfactual pair $c_i^F, c_i^T$. We will assume the reasonable scenario that the only difference among their concepts is the concept of truthfulness. That is, for any other concept $B$ that is orthogonal to $A$, the valuations of these pairs are identical.
Suppose we stack all other (orthogonal) concepts of relevance for this particular counterfactual sample into a matrix $B_i$, then we have
\begin{align*}
    Ah(c_i^F) = b_0, Ah(c_i^T) = b_1, \qquad B_ih(c_i^F) = B_ih(c_i^T)
\end{align*}
for all $i = 1, 2, \ldots$.

Now, it's clear that for a single sample $i$, we cannot choose $\eta_i = h(c_i^T) - h(c_i^F)$ as a steering vector. This is because for such a choice of steering vector, while it's still the case that $B_i(h(c_i^F) + \eta_i) = B_i(h(c_i^T) + \eta_i)$, it need not be the case that $B_j(h(x_j^F) + \eta_i) = B_j(h(x_j^T) + \eta_i)$ for some $j \neq i$, i.e., not all orthogonal concepts are preserved. This happens because $\eta_i = h(c_i^T) - h(c_i^F)$ not only contains the correct shift along the span of $A$ but also unnecessarily shifts along the space orthogonal to $B_i$, which will affect $B_j$.

However, when we have a large enough number of contexts $j$ with sufficient concept diversity, we can make the reasonable prediction that since the context questions $q_j$ are sufficiently diverse (therefore $B_j$ are diverse), the corresponding projections $B_j \eta_i$ are essentially uncorrelated vectors as $i$ varies. In this case, we can pick $\eta$ to be the mean of all the shift vectors $\eta_i$ to achieve the desired effect. Therefore, under the above simplifying assumptions, choosing $\eta$ to be the mean of $\eta_1, \eta_2, \ldots$ will satisfy
\begin{align*}
    Ah(x) &= b_0, A(h(x) + \eta) = b_1 & \text{ exact equalities for truth concept $A$}\\
    Bh(x) &\approx B(h(x) + \eta)& \text{approximate equalities for concepts $B$ orthogonal to $A$}
\end{align*}
for any concept $B$ orthogonal to $A$.
This explains why the choice of mean of the activation differences across counterfactual pairs is a reasonable choice of steering vector. This is precisely the technique used in ITI. While they also experiment with other steering vectors, they found that this works the best for their experiments.

Now, we will continue on our insights to analyze whether we can build better steering vectors $\eta$. We present two crucial insights based on our analysis so far.
\begin{enumerate}
    \item Looking at our desired equations, any \textit{weighted combination} of $\eta_i = h(c_i^T) - h(c_i^F)$ will satisfy $Ah(x) = b_0, A(h(x) + \eta) = b_1$ exactly.
    \item We could potentially choose the steering vector $\eta$ to be a function of $x$ instead of being a constant vector, provided $\eta(x)$ is efficiently computable during inference time.
\end{enumerate}

Exploiting our first insight, we conclude that choosing any weighted combintaion of the $\eta_i$ should be a reasonable choice of steering vector provided we can control its effects on the spaces orthogonal to $A$. That is, we can choose
\begin{align*}
\eta &= \sum_i w_i \eta_i= \sum_i w_i(h(c_i^T) - h(c_i^F))
\end{align*}
as our steering vector. This gives us the extra freedom to tune the weights $w_1, w_2, \ldots$ based on other heuristics. Note that this also captures the choice of the top principal component of the steering vector as experimented in \cite{tigges2023linear}.

Our second observation suggests that
even the steering vector $\eta$ could be a function of $x$, namely $\eta(x)$, provided it's efficiently computable during inference. Therefore, this suggests the usage of
\begin{align*}
\eta(x) &= \sum_i w_i(x)(h(c_i^T) - h(c_i^F))
\end{align*}
as our steering vector where the weights $w_i(x)$ depend on $x$.

Based on these two observations, we propose our ITI  modification. We choose the steering vector to be dependent on the context $x$, with weights chosen to be $w_i = \ip{\lambda(x)}{\lda(c_i^F)}$ for a sentence embedding $\lda$ (such as Sentence-BERT \cite{reimers-2019-sentence-bert}). That is,
\begin{align*}
\eta(x) &= \sum_i\ip{\lda(x)}{\lda(c_i^F)}(h(c_i^T) - h(c_i^F))
\end{align*}

Indeed, this is reasonable as if a context $x$ is close to $c_i^F$ for a specific training sample $i$ in terms of their sentence embeddings $\lda(x)$ and $\lda(c_i^F)$, then this particular sample's steering vector should be upsampled. In other words, we can think of the training sample contexts as voting on their respective counterfactual steering vector, with weights determined by the similarity between the representation of the test context and the representation of the sample context.
A justification would be that $B_i(x)$ (the relevant concepts for a datapoint) depend smoothly on $x$
(proximity is measured by similarity of embeddings) so it makes sense to upweight close points.

Finally, we need to argue that we can compute this efficiently during inference. For this, we exploit the structure of our steering vector representation as follows.
\begin{align*}
\eta(x) &= \sum_i \ip{\lda(x)}{\lda(c_i^F)}(h(c_i^T) - h(c_i^F))\\
&= \bigg(\sum_i(h(c_i^T) - h(c_i^F))\lda(c_i^F)'\bigg)h(x)\\
&= Mh(x)
\end{align*}
for the matrix $M = \sum_i(h(c_i^T) - h(c_i^F))\lda(c_i^F)'$, where $v'$ denotes the tranposed vector. We remark that the weights $w_i(x)$ as used could potentially be negative but this is not an issue since the projection of the corresponding counterfactual vector in the direction of $B$ is still random and we finally normalize $\eta(x)$, so the magnitude doesn't matter.

Therefore, this steering can be done efficiently by precomputing the \textit{steering matrix} $M$ and then during inference, we simply compute the steering vector $\eta(x)$ as $\eta(x) = Mh(x)$.

\paragraph{Implementation considerations}

We briefly note down some design choices we made in our implementation of the above method.

\begin{enumerate}
    \item Since $\eta(x)$ is a function of $x$, the standard deviation of the activation projection on this direction, i.e., $\sig_l^h(x)$ cannot be precomputed (as \citet{li2023inference} do), therefore we compute them dynamically during inference, which takes little overhead with fast tensorization operations (in particular, this is not the slow step).
    \item We opted to go with evaluating the model only on the multiple-choice questions. This is partly because to evaluate the generated text, the recommended method is to use fine-tuned GPT-3-13B models but OpenAI have retired many of their older models as of this year, and therefore, the entire batch of experiments would have to be rerun with their newer models which could potentially change the baselines, and also because this work is a proof-of-concept rather than a comprehensive evaluation.
    \item For computing the sentence embeddings, we only use the question prompts, as they contain all relevant contexts. And we normalize $\eta(x)$ during inference time.
\end{enumerate}

\paragraph{Additional ideas for improvement}
We re-iterate that our experimental exploration is not exhaustive and the preliminary experiments are merely meant to be a proof-of-concept.
In this section, building on our insights, we outline some further ideas to improve the performance of ITI.
We leave to future work to comprehensively explore these techniques in order to extract better performance towards LLM alignment.
\begin{enumerate}
    \item Note that we opted to go with the weights $\ip{\lda(x)}{\lda(c_i^F)}$ where $\lda$ was chosen to be a sentence transformer embedding \cite{reimers-2019-sentence-bert}. While this is a reasonable choice, similarity metrics could be measured in other ways, e.g., with other sentence embedding models.
    \item Going further, the weights do not have to be similarity scores and could be chosen via other heuristics. For instance, they could be chosen to be constants but potentially be optimized using a hold-out test set.
    \item As \citet{li2023inference} noted, the ITI technique could be applied on top of fine-tuned models in order to further improve their performance. Therefore, our proposed modification could also potentially be applied on top of fine-tuned models.
\end{enumerate}

\section{Contrastive algorithm for end-to-end concept learning}
\label{sec: contrastive}

In this section, we will describe the details of the contrastive learning method for learning concept representations end-to-end.
First, we will prove the computation of the true log-odds.

\logodds*
\begin{proof}
    This follows from \cref{eqn: logodds} in the proof of \cref{thm:ivae}.
\end{proof}

From our main identifiability results, we can assume  without loss of generality that the concept vectors we learn are coordinate vectors.
In other words, we consider a neural network $h^{\theta}$ with parameters $\theta$ with output neurons $h^{\theta}_1, \ldots, h^{\theta}_n$ such that the $n$ atomic concepts will now correspond to the concept vectors $e_1, \ldots, e_n$ (which is reasonable as they are only identifiable up to linear transformations).
Therefore, for each environment $e$, we can train classifiers of the form
\begin{align*}
    g_e(X, \al^e, \beta^e_k, \gam^e_k, \theta) &=  \al^e - \sum_{k=1}^{\dim(C_e)} (\beta_e^kh^{\theta}_k(X))^2 + \sum_{k=1}^{\dim(C_e)} \gam_e^k(h^{\theta}_k(X))
\end{align*}
equipped with standard cross-entropy loss,
for hyperparameters $\al^e, \beta^e_k, \gam^e_k, \theta$.
Indeed, this is reasonable since if the training reaches the global optima in the ideal case, then the loss function will correspond to the Bayes optimal classifier and therefore, $g_e(X, \al^e, \beta^e_k, \gam^e_k, \theta) = \ln(p^e(Z))- \ln(p(Z))$, which along with \cref{lem: logodds} will suggest that the learnt network $h$ is linearly related to the function $A^ef^{-1}$, as desired. Lastly, we choose the loss function to be the aggregated CE loss and an extra regularization term. That is,
\begin{align*}
    \calL = \sum_{e} \underbrace{-\EE_{j \sim \text{Unif}(\{0, e\})} \EE_{X \sim X^e} \left(\ln \frac{e^{\boldsymbol{1}_{j = e}g_e(X)}}{1 + e^{g_e(X)}}\right)}_{\text{CE loss for environment $e$}} \quad+ \quad \eta \norm{\beta}_1
\end{align*}
for a regularization hyperparameter $\eta$.

\section{Additional details about the synthetic setup}
\label{sec: synth_setup}

In this section, we detail our synthetic setup in \cref{sec: synthetic}.
The base distribution is sampled from a Gaussian mixture model with 3 components
whose parameters are chosen randomly.
The weights are randomly chosen from Unif$(0.3, 1)$ (and then normalized), the entries of the means are chosen from Unif$(-1, 1)$ and the covariance is chosen to be a diagonal matrix with entries in Unif$(0.01, 0.015)$ (note that the diagonal nature doesn't really matter since a map $f$ will be applied to this distribution).
The mixing function $f$ is chosen to be either (i) linear or (ii) nonlinear with a 1-layer MLP containing 16 hidden neurons and LeakyReLU$(0.2)$ activations.

The number of concepts $n$ is intentionally chosen to be less than the ground truth dimension $d_z$ and the number of concepts is $m = n + 1$ as per our theory.
The concepts are taken to be atomic, with the concept vectors and valuations chosen randomly, where each entry of the concept vector is chosen i.i.d from Unif$(-0.3, 0.3)$, and the resampling distribution is chosen to be a Gaussian with variance $0.005$.
Finally, we choose $5000$ samples per environment, sampled via the rejection sampling Algorithm \ref{algo: sampling}.
For the contrastive algorithm, we choose the architecture to either be linear or nonlinear with a 2-layer MLP with 32 hidden neurons in each layer, with the final parametric layer chosen based on the known concept, to have the form described above.
We train for 100 epochs with $\eta = 0.0001$ and use Adam optimizer with learning rates $0.5$ for the parametric layer and $0.005$ for the non-parametric layer, with a Cosine Annealing schedule \cite{cosine}.

\section{Controllable generative modeling via rejection sampling}
\label{sec: sampling}

In this section, we will describe how to sample from a concept conditional distribution with a known concept. Once the concepts are learned in our framework, we can use this technique to generate new data satisfying various desired concepts, which will aid in controllable generative modeling.

Consider the base distribution on $Z \in \RR^{d_z}$ with density $p(Z)$. Suppose we wish to sample from a concept $C$ given by $AZ = b$ and resampling distribution $q$. We additionally assume that $q$ is efficiently computable and an upper bound $L$ is known for its density, i.e., $L \ge \max(q)$.

Recall that the desired density is defined as
\[p_C(Z) \propto p(Z)
 \prod_{i \le dim(C)} q((AZ - b)_i)\]

Note that it's infeasible to compute the normalization constant for such complex distributions. However, we bypass this by using rejection sampling. We describe the procedure in Algorithm \ref{algo: sampling}.

\begin{algorithm}[!ht]
\label{alg: sampling}
\DontPrintSemicolon
\KwIn{
\begin{itemize}
    \item Base distribution $p$
    \item Resampling distribution $q$ with upper bound $L \ge \max(q)$
    \item Concept $C$ with transformation $A$ and valuation $C$
\end{itemize} }
\KwOut{Returns a single sample from $p_C(Z)$}
$M = L^{dim(C)}$\\
\tcp{Repeat trials until condition is met}
\While{True}{
Z = yield(p)\\
U = yield(Unif(0, 1))\\
$R = \frac{1}{M}\prod_{i \le dim(C)} q((AZ - b)_i)$\\
\If{$R \ge U$}{\Return{$Z$}}
}
\caption{Rejection sampling for controllable generative modeling}
\label{algo: sampling}
\end{algorithm}

Informally, we first sample $Z \sim p$ (we overload notation for both density and the distribution) and an independent variable $U \sim Unif(0, 1)$, the uniform distribution on $(0, 1)$. We accept the variable $Z$ if
\[\frac{1}{M}\prod_{i \le dim(C)} q((AZ - b)_i) \ge U\]
for a predetermined upper bound $M$ on the quantity $\prod_{i \le dim(C)} q((AZ - b)_i)$. If the inequality is false, we simply reject the sample and repeat.

Now we will argue why this algorithm is correct, which is accomplished in \cref{thm: proof_algo_sampling}.
Let
\[N_C = \int_{Z} p(Z)\prod_{i \le dim(C)} q((AZ - b)_i)\]
be the normalization constant in the definition of $p_C(Z)$. Therefore
\[p_C(Z) = \frac{1}{N_C}p(Z)
 \prod_{i \le dim(C)} q((AZ - b)_i)\]

\begin{lemma}
\label{lem: acceptance_prob}
    Let $M \ge \max(q)^{dim(C)}$
    The acceptance probability of each iteration of the while loop in Algorithm \ref{algo: sampling} is $Pr[Z{\text{ accepted}}] = \frac{N_C}{M}$
\end{lemma}

\begin{proof}
We have
\begin{align*}
    Pr[Z{\text{ accepted}}] &= Pr_{U, Z}\left[U \le \frac{1}{M}\prod_{i \le dim(C)} q((AZ - b)_i)\right]\\
    &= Pr_{U, Z}\left[U \le \prod_{i \le dim(C)} \frac{q((AZ - b)_i)}{\max(q)}\right] && \text{ since $M \ge max(q)^{dim(C)}$}\\
    &= \int_Z Pr_U\left[U \le \prod_{i \le dim(C)} \frac{q((AZ - b)_i)}{\max(q)}\right] p(Z)\, dZ&& \text{as $U, Z$ are independent}\\
    &= \int_Z \left[\prod_{i \le dim(C)} \frac{q((AZ - b)_i)}{\max(q)}\right] p(Z)\, dZ && \text{ since $\frac{q((AZ - b)_i)}{\max(q)} \le 1$ always}\\
    &= \int_Z \frac{N_C p_C(Z)}{M}\, dZ\\
    &= \frac{N_C}{M}
\end{align*}
\end{proof}

Before we prove correctness, we will remark on the expected number of trials needed for accepting each sample.

\begin{corollary}
    The expected number of trials needed to generate a single sample is $\frac{M}{N_C}$
\end{corollary}

\begin{proof}
    Note that each iteration of the while loop is independent, therefore the number of trials until acceptance is distributed as a geometric random variable whose expectation is the inverse of the parameter.
\end{proof}

This suggests that for our algorithm to be efficient in practice, $M$ should be chosen as small as possible, i.e., estimates of $\max(q)$ should be as tight as possible.

\begin{theorem}
\label{thm: proof_algo_sampling}
    Algorithm \ref{algo: sampling} yields samples from the concept conditional distribution $p_C$.
\end{theorem}

\begin{proof}
The proof is at heart the proof of correctness of rejection sampling.
For arbitrary parameters $t_1, \ldots, t_{d_z} \in \RR$, let's compute the cumulative density of the samples output by Algorithm \ref{algo: sampling} and show that it matches the cumulative distribution function of $p_C(Z)$ evaluated at $t_1, \ldots, t_{d_z}$, which will complete the proof. That is, we wish to calculate \begin{align*}
    Pr[Z_1 \le t_1, \ldots, Z_{d_z} \le t_{d_z} | Z \text{ accepted}]
    &= \frac{Pr[Z_1 \le t_1, \ldots, Z_{d_z} \le t_{d_z}, Z\text{ accepted}]}{Pr[Z{\text{ accepted}}]}
\end{align*}
We already computed the denominator in \cref{lem: acceptance_prob}.
Therefore,
\begin{align*}
    Pr&[Z_1 \le t_1, \ldots, Z_{d_z} \le t_{d_z} | Z \text{ accepted}]\\
    &= \frac{M}{N_C}Pr[Z_1 \le t_1, \ldots, Z_{d_z} \le t_{d_z}, Z\text{ accepted}]\\
    &= \frac{M}{N_C} \EE_Z \left[\one_{Z_1 \le t_1} \ldots \one_{Z_{d_z} \le t_{d_z}}\cdot\EE_U[\one_{Z \text{ accepted}}]\right]\\
    &= \frac{M}{N_C} \EE_Z \left[\one_{Z_1 \le t_1} \ldots \one_{Z_{d_z} \le t_{d_z}}\cdot\frac{1}{M} \prod_{i \le dim(C)}q((AZ - b)_i)\right] && \text{ from the proof of \cref{lem: acceptance_prob}}\\
    &= \int_Z \one_{Z_1 \le t_1} \ldots \one_{Z_{d_z} \le t_{d_z}}\cdot \frac{1}{N_C} \prod_{i \le dim(C)}q((AZ - b)_i) p(Z)\, dZ\\
    &= \int_Z \one_{Z_1 \le t_1} \ldots \one_{Z_{d_z} \le t_{d_z}}\cdot p_C(Z)\, dZ
\end{align*}
which is precisely the cumulative distribution function of $p_C(Z)$ evaluated at $t_1, \ldots, t_{d_z}$.
\end{proof}

\end{document}